\documentclass[times,twocolumn,final,authoryear]{elsarticle}
\usepackage{framed,multirow}

\usepackage{amssymb}
\usepackage{latexsym}

\usepackage{url}
\usepackage{xcolor}
\definecolor{newcolor}{rgb}{.8,.349,.1}

\usepackage{amsmath}
\usepackage{array}
\usepackage{url}
\usepackage{graphicx}
\usepackage{amsmath}
\usepackage{amsfonts} 
\usepackage{amssymb}
\usepackage{amsthm}
\usepackage{color}
\usepackage{times}
\usepackage[utf8]{inputenc}
\usepackage{version}
\usepackage[tight, footnotesize]{subfigure}
\usepackage{relsize}
\usepackage{paralist}
\usepackage[super]{nth}
\usepackage{capt-of,etoolbox}
\usepackage{multirow}
\usepackage{xcolor}
\usepackage{algorithm}
\usepackage[noend]{algpseudocode}
\DeclareMathOperator{\argmin}{argmin}

\newtheorem{proposition}{Proposition}

\newcommand{\TT}{\mathcal{T}}

\newcommand{\SSS}{\mathcal{S}}
\newcommand{\CC}{\mathcal{C}}
\newcommand{\PP}{\mathcal{\mathbf{P}}}

\begin{document}

 \thispagestyle{empty}

\ifpreprint
  \setcounter{page}{1}
\else
  \setcounter{page}{1}
\fi

\begin{frontmatter}

\title{Robust superpixels using color and contour features along linear path}

\author[1,2,3,4]{R{\'e}mi {Giraud}\corref{cor1}} 
\author[1,2,5]{Vinh-Thong {Ta}}
\author[3,4]{Nicolas {Papadakis}}

\address[1]{Univ. Bordeaux, LaBRI, UMR 5800, PICTURA, F-33400 Talence, France.}
\address[2]{CNRS, 		 LaBRI, UMR 5800, PICTURA, F-33400 Talence, France.}
\address[3]{Univ. Bordeaux, 	IMB, UMR 5251, F-33400 Talence, France.}
\address[4]{CNRS,  		IMB, UMR 5251, F-33400 Talence, France.}
\address[5]{Bordeaux INP, LaBRI, UMR 5800, PICTURA, F-33400 Talence, France.}


\begin{abstract}

Superpixel decomposition methods are widely used in computer vision and image processing applications.
By grouping homogeneous pixels, the accuracy can be increased and the decrease of the number 
of elements to process can drastically reduce the computational burden.
For most superpixel methods, a trade-off is computed between 1) color homogeneity, 
2) adherence to the image contours and 3) shape regularity of the decomposition.
In this paper, we propose a framework that 
jointly enforces all these aspects and
provides accurate and regular Superpixels with Contour Adherence using Linear Path (SCALP).
During the decomposition, we propose to consider color features
along the linear path between 
the pixel and the corresponding superpixel barycenter.
A contour prior is also used to prevent the crossing of image boundaries
when associating a pixel to a superpixel. 
Finally, in order to improve the decomposition accuracy and the robustness to noise,   
we propose to integrate the pixel neighborhood information,  
while preserving the same computational complexity.
SCALP is extensively evaluated on standard segmentation dataset,
and the obtained results 
outperform the ones of the 
state-of-the-art methods.
SCALP is also extended for supervoxel decomposition on MRI images.
\end{abstract}


 \end{frontmatter}

 \newpage

\section{Introduction}

The use of superpixels has become a very popular technique 
for many computer vision and image processing applications such as:
  object localization \citep{fulkerson2009},
   contour detection \citep{arbelaez2011}, 
   face labeling \citep{kae2013},
data associations across views \citep{sawhney2014}, 
 or multi-class object segmentation 
\citep{giraud2017spm,gould2008,gould2014,tighe2010,yang2010}. 
Superpixel decomposition methods group pixels into homogeneous regions, 
providing a low-level representation that tries to respect the image contours.
For image segmentation, 
where the goal is to split the image into similar regions according to object, color or texture priors, 
the decomposition into superpixels may improve the segmentation accuracy and decrease the computational burden. \citep{gould2014}.
Contrary to multi-resolution approaches, that decrease the image size, superpixels
preserve the image geometry, since their boundaries follow the image contours.
Hence, the results obtained at the superpixel level may be closer to the ground truth result at the pixel level.

Many superpixel methods have been proposed using various techniques. 
Although the definition of an optimal decomposition depends on the tackled application,
most methods tend to achieve the following properties.
First, the boundaries of the decomposition should adhere to the image contours, and
superpixels should not overlap with multiple objects.
Second, the superpixel clustering must  
group pixels 
with homogeneous colors.
Third, the superpixels should have compact shapes and consistent sizes.
The shape regularity helps to visually analyze the image decomposition and
has been proven to impact application performances \citep{reso2013,veksler2010,strassburg2015influence}.
Finally, since superpixels are usually used as a pre-processing step,
the decomposition should be obtained in limited computational time and 
allow the control of the number of produced elements.

To achieve the aforementioned properties, 
most state-of-the-art methods compute a trade-off 
between color homogeneity and shape regularity of the superpixels.
Nevertheless, some approaches less consider the regularity property and
can produce superpixels of highly irregular shapes and sizes.
{\color{black}
In the following, we present an overview of the most popular superpixel methods, defined as either  
irregular or regular ones.
Note that although some methods can include terms into their models to generate for instance more regular results,
\emph{e.g.}, \cite{vandenbergh2012},
we here consider methods in their default settings, as described by the authors.

The regularity criteria can be seen as the behavior to frequently 
produce irregular regions, in terms of both shapes and sizes \citep{giraud_jei_2017}. 
Methods such as \citet{felzenszwalb2004,vedaldi2008} generate very irregular regions in terms of both 
size and shape while SLIC can generate a few irregular shapes
but their sizes are constrained into a fixed size window.
}

\subsection*{Irregular Superpixel Methods}

With irregular methods, superpixels can have very different sizes and stretched shapes. 
For instance, small superpixels can be produced, without 
enough pixels to compute a significant descriptor.
Too large superpixels may also overlap with several objects contained in the image.
First segmentation methods, such as the watershed approach, \emph{e.g.}, \citet{vincent91},
compute decompositions of highly irregular size and shape.
Methods such as Mean shift \citep{comaniciu2002} or Quick shift \citep{vedaldi2008}  
consider an initial decomposition and perform a histogram-based segmentation.
However, they are very sensitive to parameters and are obtained with high computational cost \citep{vedaldi2008}. 
Another approach considers pixels as nodes of a graph to perform a faster agglomerative clustering \citep{felzenszwalb2004}.
These methods present an important drawback:
they do not allow to directly control the number of superpixels.
This is particularly an issue when superpixels are used as a low-level representation
to reduce the computational time.

{\color{black}
The SEEDS method \citep{vandenbergh2012} 
proposes a coarse-to-fine approach starting from a regular grid.
However, this method may provide superpixels with irregular shapes.
Although  a compactness constraint can be set to compute regular superpixels,
the authors report degraded results of decomposition accuracy with such approach.
}

\subsection*{Regular Superpixel Methods}

For superpixel-based object recognition methods, \emph{e.g.}, \citet{gould2008,gould2014}, 
or video tracking, \emph{e.g.}, \citet{reso2013,wang2011},
the use of regular decompositions is mandatory, \emph{i.e.}, 
decompositions with superpixels having approximately the same size and compact shapes.
For instance, for superpixel-based video tracking applications,
the tracking of object trajectories within a scene 
is improved with consistent decompositions over time \citep{chang2013,reso2013}.

Most of the regular methods consider an initial regular grid, allowing to set the number of superpixels,
and update superpixels boundaries 
while applying spatial constraints.
Classical methods are based on region growing, such as Turbopixels \citep{levinshtein2009} using geometric flows,
or eikonal-based methods, \emph{e.g.}, ERGC \citep{buyssens2014}, while 
other approaches use graph-based energy models \citep{liu2011,veksler2010}.
In  \citet{machairas2015}, a 
watershed algorithm is adapted to produce regular decompositions
using a spatially regularized image gradient. 
Similarly to SEEDS \citep{vandenbergh2012}, a coarse-to-fine approach has recently been proposed in \citet{yao2015},
producing highly regular superpixels.

The SLIC method \citep{achanta2012}
performs an iterative accurate clustering, while providing regular superpixels,
in order of magnitude faster than graph-based approaches \citep{liu2011,veksler2010}.
The SLIC method has been extended in several recent works, \emph{e.g.}, \citet{chen2017,huang2016,rubio2016,zhang2016,zhang2017ssgd}.
However, it can fail to adhere to image contours, as for other regular methods, \emph{e.g.}, \citet{levinshtein2009,yao2015}, 
since it is based on simple local color features and 
globally enforces the decomposition regularity using a 
fixed trade-off between color and
spatial distances.

\subsection*{Contour Constraint}

In the literature, several works have attempted to improve the decomposition performances in terms of contour adherence
by using gradient or contour prior information.
In \citet{mori2004}, a contour detection algorithm is  
used to compute a pre-segmentation
using the normalized cuts algorithm \citep{shi2000}. 
The segmentation  
may accurately guide the superpixel decomposition, but such
approaches based on normalized cuts are  
computationally expensive \citep{mori2004}.
Moreover, the contour adherence of the produced decompositions are far from state-of-the-art results \citep{achanta2012}.  
In \citet{moore2008}, the superpixel decomposition is constrained to fit to a grid, 
also called superpixel lattice. 
The decomposition is then refined using graph cuts.
However, this method is very dependent on the used contour prior.
Moreover, although 
the superpixels have approximately the same sizes, they have quite irregular shapes and 
may appear visually unsatisfactory.

In \citet{machairas2015}, 
the image gradient information is used to constrain the superpixel boundaries, 
but the results on superpixel evaluation metrics are lower than the ones of SLIC \citep{achanta2012}.
In \citet{zhang2016}, the local gradient information is considered
to improve the superpixel boundaries evolution.
However, the computational cost of the method is increased by a $10{\times}$ order of magnitude compared to SLIC.

\subsection*{Segmentation from Contour Detection}

Contour detection methods generally do not enforce the contour closure.
To produce an image segmentation,
a contour completion step is hence necessary.
Many contour completion methods have been proposed 
(see for instance \citet{arbelaez2011} and references therein).
This step may improve the accuracy of the contour detection, 
since objects are generally segmented by closed curves.

Methods such as \citet{arbelaez2008,arbelaez2009}, 
propose a hierarchical image segmentation based on contour detection.
This can be considered as a probability contour map,  
that produces a set of closed curves for
any threshold. 
Although such methods enable to segment an image from a contour map,
they do not allow to control the size, the shape and the number of the produced regions,
while most superpixel methods enable to set the number of superpixels which approximately 
have the same size.
Moreover, the performances of the contour detection
is extremely dependent on the fixed threshold parameter,
which depends on the image content \citep{arbelaez2009}.
Hence, they are mainly considered as segmentation methods and 
cannot be considered as relevant frameworks to compute superpixel decompositions.

\subsection*{Robustness to Noise}

Superpixel decompositions are usually used as a pre-processing step in many computer vision applications.
Therefore, they tend to be applied to heterogeneous images that can suffer from noise.
{\color{black}Moreover, image textures and high local gradients may also mislead the superpixel decomposition.
Most of the state-of-the-art superpixel methods are not robust to noise, and provide 
degraded decompositions when applied to slightly noised images or images with low resolution.}
With such approaches, a denoising step is necessary to compute a relevant decomposition.
For instance, the watershed approach of \citet{machairas2015} 
uses a pre-filtering step to smooth local gradients according to the
given size of superpixels.
Nevertheless, this step is only designed to smooth local gradients of initial images and
the impact of this filtering is not reported \citep{machairas2015}.

\newcommand{\wo}{0.1175\textwidth}
\newcommand{\ho}{0.34\textwidth}
\begin{figure*}[t!]
\centering
{\scriptsize
\begin{tabular}{@{\hspace{0mm}}c@{\hspace{1mm}}c@{\hspace{1mm}}c@{\hspace{1mm}}c@{\hspace{1mm}}c@{\hspace{1mm}}c@{\hspace{1mm}}c@{\hspace{1mm}}c@{\hspace{0mm}}}
\includegraphics[width=\wo]{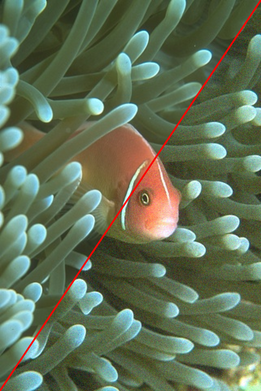}&
\includegraphics[width=\wo]{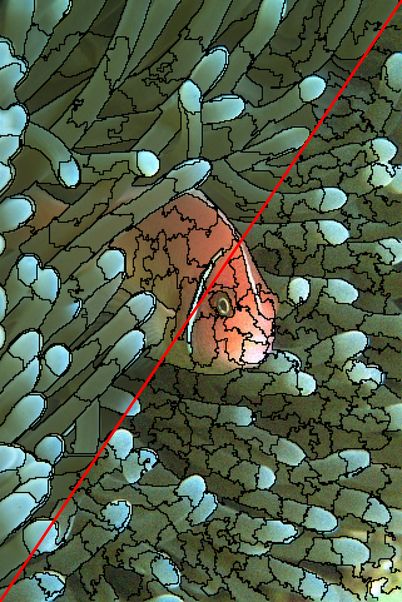}&
\includegraphics[width=\wo]{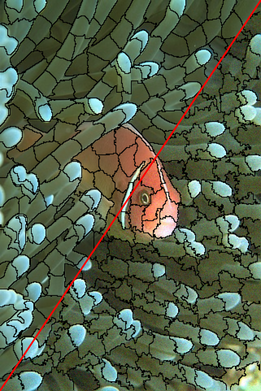}&
\includegraphics[width=\wo]{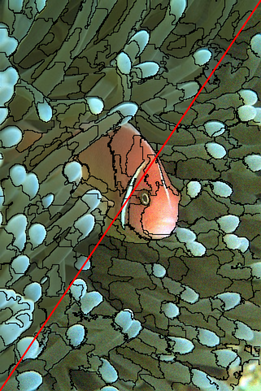}&
\includegraphics[width=\wo]{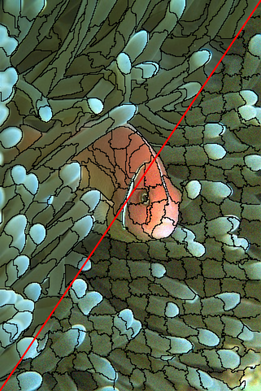}&
\includegraphics[width=\wo]{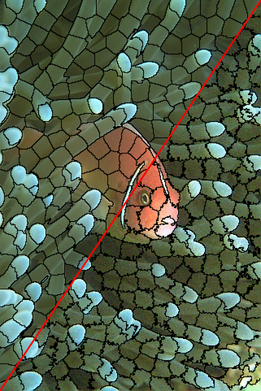}&
\includegraphics[width=\wo]{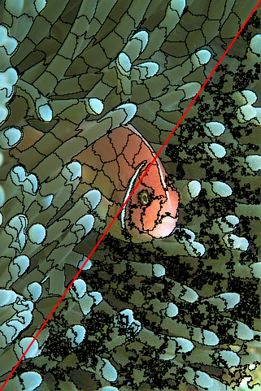}&
\includegraphics[width=\wo]{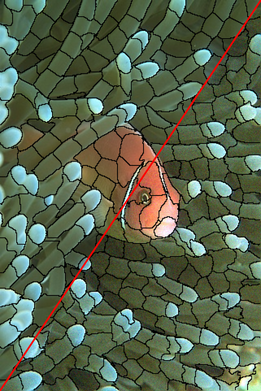}\\
Initial/noisy image & ERS & SLIC & SEEDS & ERGC &ETPS &LSC & SCALP\\
\end{tabular}
}
\caption{Comparison of the proposed SCALP approach to the following state-of-the-art superpixel methods:
ERS \citep{liu2011},
SLIC \citep{achanta2012},
SEEDS \citep{vandenbergh2012},
ERGC \citep{buyssens2014},
ETPS \citep{yao2015} and
LSC \citep{chen2017}.
SCALP obtains the most visually satisfying result with superpixels that adhere well to the image contours.
A Gaussian noise has been added to the bottom-right part of the image to demonstrate that SCALP is robust to noise,
contrary to most of the compared methods.
} 
\label{fig:reg} 
\end{figure*}

\subsection*{Contributions}

In this paper, we propose a method that produces accurate, regular and robust 
Superpixels with Contour Adherence using Linear Path (SCALP)\footnote{An implementation of the proposed SCALP method is available at: 
\url{www.labri.fr/~rgiraud/research/scalp.php}}.
Our decomposition approach aims to jointly improve all superpixel properties:
color homogeneity, respect of image objects and shape regularity.
In Figure \ref{fig:reg}, we compare the proposed approach to state-of-the-art methods on an example result.
SCALP provides a more satisfying result that respects the image contours.
Moreover, contrary to most state-of-the-art methods, 
SCALP is robust to noise, 
since it provides accurate and regular decompositions 
on the noisy part of the image.
\begin{itemize}
 \item Most state-of-the-art methods have very degraded performances when applied to even slightly noised images (see Figure \ref{fig:reg}).
We propose to consider the neighboring pixels information during the decomposition process.
We show that these features can be integrated at the same computational complexity, while they 
improve the decomposition accuracy and the robustness to noise.
\item
To further enforce the color homogeneity within a regular shape,
we define the linear path between the pixel and the superpixel barycenter,
and we consider color features along the path. 
Contrary to geodesic distances that can allow irregular paths 
leading to non convex shapes, our linear path naturally enforces the decomposition regularity.
A contour prior can also be used 
to enforce the respect of image objects and prevent the crossing of image contours when associating a pixel to a superpixel.
\item We propose a framework to generate superpixels within an initial segmentation 
computed from a contour prior completion.
The produced superpixels are regular in terms of size and shape although they are constrained by the segmentation 
to obtain higher contour adherence performances.
\item We provide an extensive evaluation of SCALP on the Berkeley segmentation dataset (BSD).
Our results outperform recent state-of-the-art methods, on initial and noisy images, 
in terms of superpixel and contour detection metrics. 
\item Finally, we naturally extend SCALP to supervoxel decomposition and provide results
on magnetic resonance imaging (MRI) segmentation.
\end{itemize}

This paper is an extension of the work proposed in \citet{giraud2016}, 
with substantial new improvements such as 
the use in constant time of the neighboring pixels information,
the use of contour prior by considering the maximum intensity on the linear path,
or the extension to supervoxels.
We show that these new contributions improve the decomposition performances,
and by performing the clustering in a high dimensional feature space \citep{chen2017},
SCALP substantially outperforms \citet{giraud2016} and the recent state-of-the-art methods.

\section{SCALP Framework\label{sec:scalp}}

The SCALP framework is based on
 the simple linear iterative clustering framework (SLIC) \citep{achanta2012},
 and is summarized in Figure \ref{fig:scap}.
 In this section, we first present SLIC and then propose several improvements: 
 a robust distance on pixel neighborhood, the use of features along the linear path to the superpixel barycenter and
 a framework considering an initial segmentation as constraint
while producing regular superpixels.

\begin{figure}[t!]
\centering
\includegraphics[width=0.47\textwidth]{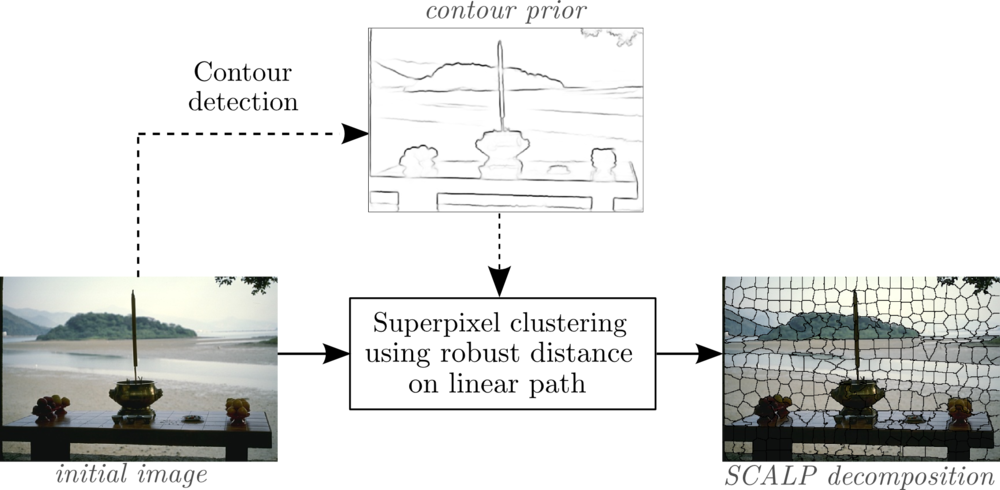}
\caption{The SCALP framework. 
A prior can be used (dotted arrows) to enforce the respect of
image contours, leading to an accurate decomposition.
When trying to associate a pixel to a superpixel, 
SCALP considers the color information from neighboring pixels, and color and contour features
on the linear path to the superpixel barycenter.
} 
\label{fig:scap}
\end{figure}

\subsection{Iterative Clustering Framework}

The iterative clustering framework introduced in \citet{achanta2012} 
proposes a fast iterative framework using simple color features (average in CIELab colorspace).
The decomposition is initialized by a regular grid with blocks of size $r{\times}r$.
This size is computed by the ratio between the number of pixels $N$ and
the number of desired superpixels $K$, such that $r=\sqrt{N/K}$. 
A color clustering is then iteratively performed into fixed windows of size $(2r+1){\times}(2r+1)$ pixels 
centered on the superpixel barycenter.
The superpixel is thus constrained into this window,  
which limits its size.
Each superpixel $S_k$ is described by  
a cluster $C_k$, that
contains the average CIELab color feature on pixels $p\in S_k$, $F_k=[l_k,a_k,b_k]$, 
and $X_k=[x_k,y_k]$, the spatial barycenter of $S_k$ such that $C_k=[F_k,X_k]$.
The iterative clustering consists, for each cluster $C_k$,  in
testing  all pixels $p=[F_p,X_p]$ within a $(2r+1){\times}(2r+1)$ pixels window centered on $X_k$, by computing 
a spatial distance $d_s$, and a color distance $d_c$:
\begin{equation}
d_s(p,C_k) = {(x_p-x_k)^2 + (y_p - y_k)^2},  \label{dist_slic}
\end{equation}
\begin{equation}
d_c(p,C_k) = {(l_p-l_k)^2 + (a_p-a_k)^2 + (b_p-b_k)^2},   \label{color_slic}
\end{equation}
\begin{equation}
D(p,C_k)=d_c(p,C_k) + d_s(p,C_k)\frac{m^2}{r^2},  \label{distance_slic} 
\end{equation}
with $m$ the regularity parameter that sets the trade-off between 
 spatial and color distances. 
High values of $m$ produce more regular superpixels,
while small values allow better adherence to image boundaries,
producing superpixels of more variable sizes and shapes.
The pixel $p$ is associated to the superpixel $S_k$ minimizing \eqref{distance_slic}.

Nevertheless, 
since a parameter $m$ is set to enforce the regularity in \eqref{distance_slic}, 
SLIC can fail to both produce regular superpixels and to adhere to the image contours.
In the following, we show how the decomposition accuracy can be improved with a more robust distance, 
by considering neighboring color features and information of pixels along the linear path to the superpixel barycenter.

\subsection{Robust Distance on Pixel Neighborhood}

Natural images may present high local image gradients or noise, 
that can highly degrade the decomposition into superpixels.
In this section, we propose to consider the pixel neighborhood
to improve both accuracy and robustness,  
and we give a method to integrate this information in the decomposition process
at a constant complexity.

\subsubsection{Distance on Neighborhood}

We propose to integrate the neighboring pixels information in our framework
when computing the clustering distance between a pixel $p$ and a cluster $C_k$.
Similarly to patch-based approaches, 
the pixels in a square area $\mathcal{P}(p)$ centered on $p$, of size $|\mathcal{P}(p)|=(2n+1)\times(2n+1)$ pixels,
 are considered in the proposed color distance $D_c$:
\begin{align}
  D_c(p,C_k) &= \sum_{q\in \mathcal{P}(p)}{(F_q-F_{C_k})^2  w_{p,q}}   . \label{patch}
 \end{align}
To be robust to high local gradients while preserving the image contours, 
we define $w_{p,q}$ such that 
$w_{p,q} = \exp{\left(-{(F_p - F_q)^2}/{(2\sigma^2)}\right)}/Z$,
 with $Z$ the normalization factor such that
 $Z=\sum_{q\in \mathcal{P}(p)} \exp{\left(-{(F_p - F_q)^2}/{(2\sigma^2)}\right)}$, 
 and $\sum_{q\in \mathcal{P}(p)}w_{p,q} = 1$.

\subsubsection{Fast Distance Computation}
 
The complexity of the proposed distance \eqref{patch} is $\mathcal{O}(N)$, with
$N=(2n+1)^2=|\mathcal{P}(p)|$, the number of pixels in the neighborhood.
We propose a method that drastically reduces the computational burden of \eqref{patch}.
Since the distance is computed between a set of pixels and a cluster, 
it can be decomposed and partially pre-computed.
\begin{proposition} 
Eq. \eqref{patch} can be computed at complexity $\mathcal{O}(1)$.  
\end{proposition}
\begin{proof}
The distance between features $F$ in \eqref{patch} reads:  
 \begin{align}
 \sum_{q\in \mathcal{P}(p)} (&F_q-F_{C_k})^2 w_{p,q} \nonumber \\
 &= \hspace{-0.1cm} \sum_{q\in \mathcal{P}(p)}{\left(F_{q}^2 + F_{C_k}^2 - 2F_qF_{C_k}\right)} w_{p,q}, \nonumber \\
 &=  \hspace{-0.1cm} \sum_{q\in \mathcal{P}(p)}{F_{q}^2w_{p,q}} + \sum_{q\in \mathcal{P}(p)}{F_{C_k}^2w_{p,q}} - 2\sum_{q\in \mathcal{P}(p)}{F_q F_{C_k}w_{p,q}}, \nonumber \\
 &=   {\mathcal{F}_p}^{(2)} + F_{C_k}^2 \sum_{q\in \mathcal{P}(p)}{{w_{p,q}}} - 2F_{C_k} \sum_{q\in \mathcal{P}(p)}{F_qw_{p,q}} , \nonumber \\
  &=   {\mathcal{F}_p}^{(2)} + F_{C_k}^2 - 2F_{C_k}{\mathcal{F}_p}^{(1)} .  \label{speed_up}
  \end{align}
{\color{black}In Eq. \eqref{speed_up}, the terms ${\mathcal{F}_p}^{(2)} = \sum_{q\in \mathcal{P}(p)}{F_{q}^2w_{p,q}}$, 
 and ${\mathcal{F}_p}^{(1)} = \sum_{q\in \mathcal{P}(p)}{F_qw_{p,q}}$,
 which only depend on the initial image, can be pre-computed at the beginning of the algorithm.
The complexity of the proposed distance $D_c$ is hence reduced to $\mathcal{O}(1)$ instead of $\mathcal{O}(N)$.}
\end{proof}

\subsection{Color and Contour Features on Linear Path}

A superpixel decomposition is considered as satisfying according to the homogeneity of the
color clustering and the respect of image contours. 
To enforce these aspects,
we propose to consider color and contour features  
on the linear path between the pixel and the superpixel barycenter.
We define the linear path $\PP^k_p$, that contains the pixels starting from
$X_p$, the position of a pixel $p$, to $X_k$, the barycenter of a superpixel $S_k$.

\begin{figure}[t!]
\centering
\includegraphics[width=0.35\textwidth]{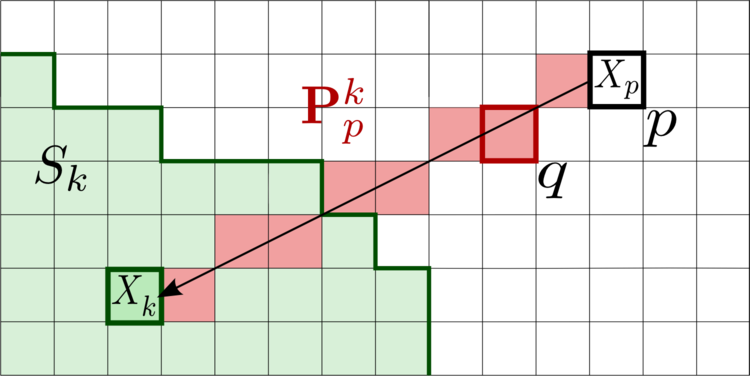} 
\caption{Illustration of the linear path $\PP_p^k$ between a pixel $p$ and
a superpixel $S_k$ of barycenter $X_k$.
}  
\label{fig:bres_diag}
\end{figure}

\subsubsection{Linear Path between Pixel and Superpixel Barycenter}

The considered linear path $\PP^k_p$ between 
a pixel $p$ and the barycenter of a superpixel $S_k$ is 
illustrated in Figure \ref{fig:bres_diag}.
The pixels $q\in \PP^k_p$ (red) are those that intersect with the segment (arrow) between $X_p$,
the position of pixel $p$ (black),
and $X_k$, the barycenter of the superpixel $S_k$ (green).
Pixels $q$ are selected such that each one only has $2$ neighbors belonging to the path 
within a $3{\times}3$ pixels neighborhood.

{\color{black}Other works consider a geodesic distance to enforce 
the color homogeneity \citep{rubio2016,wang2013} or the respect of object contours \citep{zhang2017ssgd}.}
The colors along the geodesic distance must be close to 
the average superpixel color to enable the association of the pixel to the superpixel,
leading to potential irregular shapes.
We illustrate this aspect in Figure \ref{fig:geodesic}.
We compare a geodesic distance and average color distance on the linear path.
While the geodesic can find a sinuous path to connect distant pixels, 
our linear path penalizes the crossing of regions with different colors.

{\color{black}A decomposition example for SCALP and a method based on a geodesic color distance \citep{rubio2016} is given in Figure \ref{fig:geodesic_comp}.
By considering the proposed linear path, 
we limit the computational cost,
that can be substantial for geodesic distances,
and we enforce the decomposition compactness,
since features are considered on the direct path to the superpixel barycenter.}
{\color{black}
More precisely, our linear path encourages the star-convexity property \citep{gulshan2010geodesic}, \emph{i.e.},
for a given shape, it exists a specific point,
in our case, the superpixel barycenter, 
from which each point of the shape can be reached by a linear path that does not escape from the shape.
}

{\color{black}Finally, note that despite the large number of pixel information considered during the decomposition process, 
the computational cost can be very limited.
In practice, at a given iteration, for a given superpixel, the distance between a pixel and the superpixel has only to be computed once.
{\color{black}
The color distance can indeed be stored for each pixel and directly used for another linear path containing this pixel.
Moreover, a very slight approximation can be made by directly storing for each pixel
}
the average distance on the linear path to the superpixel barycenter,
and using it when crossing an already processed pixel on a new linear path.}

\begin{figure}[t!]
\centering
{\footnotesize
\begin{tabular}{@{\hspace{0mm}}c@{\hspace{1mm}}c@{\hspace{1mm}}c@{\hspace{0mm}}}
\includegraphics[width=0.155\textwidth]{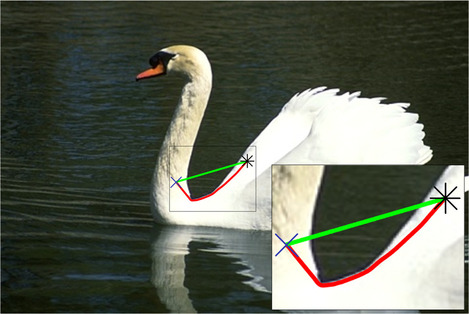}&
\includegraphics[width=0.155\textwidth]{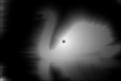}&
\includegraphics[width=0.155\textwidth]{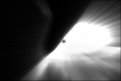}\\
(a) Image &(b) Geodesic distance & (c) Linear path distance\\
\end{tabular}
}
\caption{
Comparison of geodesic path (red) and linear path (green) between
an initial (star) and final (cross) pixel position in (a).
In (b) and (c), lighter colors indicate a lower distance from the initial point (star).
When trying to associate this point to a superpixel, 
the distance at the superpixel barycenter position is considered.
Contrary to the linear path, defined in the spatial space, 
the geodesic path, defined in the color space, may lead to irregular superpixel shapes.
}
\label{fig:geodesic}
\end{figure}

\begin{figure}[t!]
\centering
{\footnotesize
\begin{tabular}{@{\hspace{0mm}}c@{\hspace{2mm}}c@{\hspace{2mm}}c@{\hspace{0mm}}}
\includegraphics[width=0.155\textwidth]{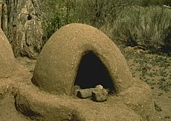}&
\includegraphics[width=0.155\textwidth]{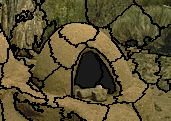}&
\includegraphics[width=0.155\textwidth]{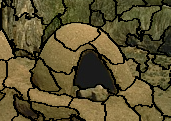}\\
(a) Image & (b) \citet{rubio2016}  & (c) SCALP  \\
 &  (geodesic) &   (linear path)
\end{tabular}
} 
\caption{
{\color{black}Comparison on an image (a) of decomposition approaches using a geodesic color distance \citep{rubio2016} (b),
and the proposed method SCALP, computing a color distance on the linear path to the superpixel barycenter (c).
SCALP generates regular shapes while the geodesic-based method can create irregular superpixels.}
}  
\label{fig:geodesic_comp}
\end{figure}

\subsubsection{Color Distance to Cluster}
The distance to minimize during the decomposition is composed of a color and a spatial term.
Nevertheless, the color distance
is now also computed on $\PP^k_p$, \emph{i.e}, between the cluster and the pixels 
on the linear path to the superpixel barycenter. 
We define the new color distance as: 
\begin{equation}
d_c(p,C_k,\PP^k_p )\hspace{-0.05cm}=\hspace{-0.05cm}\lambda D_c(p,C_k)\hspace{-0.1mm}+\hspace{-0.1mm}(1\hspace{-0.05cm}-\hspace{-0.05cm}\lambda)
\frac{1}{|\PP^k_p|}\hspace{-0.1cm}\sum_{q\in \PP^k_p}\hspace{-0.1cm}D_c(q,C_k), 
\label{newdist0}
\end{equation}
\noindent where 
$\lambda\in[0,1]$ weights the influence of the color distance along the path. 
With the proposed distance \eqref{newdist0}, 
colors on the path to the barycenter should be close to the superpixel average color.

The distance \eqref{newdist0} naturally enforces the regularity and
also prevents irregular shapes to appear.
Figure \ref{fig:banana} shows two examples of irregular shapes
that can be computed with SLIC \citep{achanta2012}, for instance
in areas of color gradation.
The barycenters $X_k$ of these irregular superpixels $S_k$ are not contained within the shapes.
The linear path $\PP^k_p$ hence capture pixels with colors that are far from the average one of $S_k$.
Therefore, \eqref{newdist0} penalizes the clustering of all pixels $p\in S_k$ 
to this superpixel during the current iteration,
so they are associated to neighboring superpixels.

\begin{figure}[h!]
\centering
{\footnotesize
\begin{tabular}{@{\hspace{0mm}}c@{\hspace{1mm}}c@{\hspace{1.5mm}}c@{\hspace{1mm}}c@{\hspace{0mm}}}
\includegraphics[height=0.075\textwidth,width=0.115\textwidth]{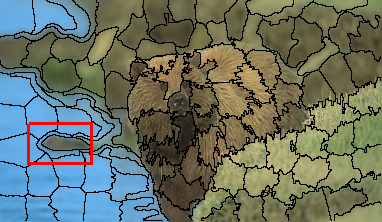}&
\includegraphics[height=0.075\textwidth,width=0.115\textwidth]{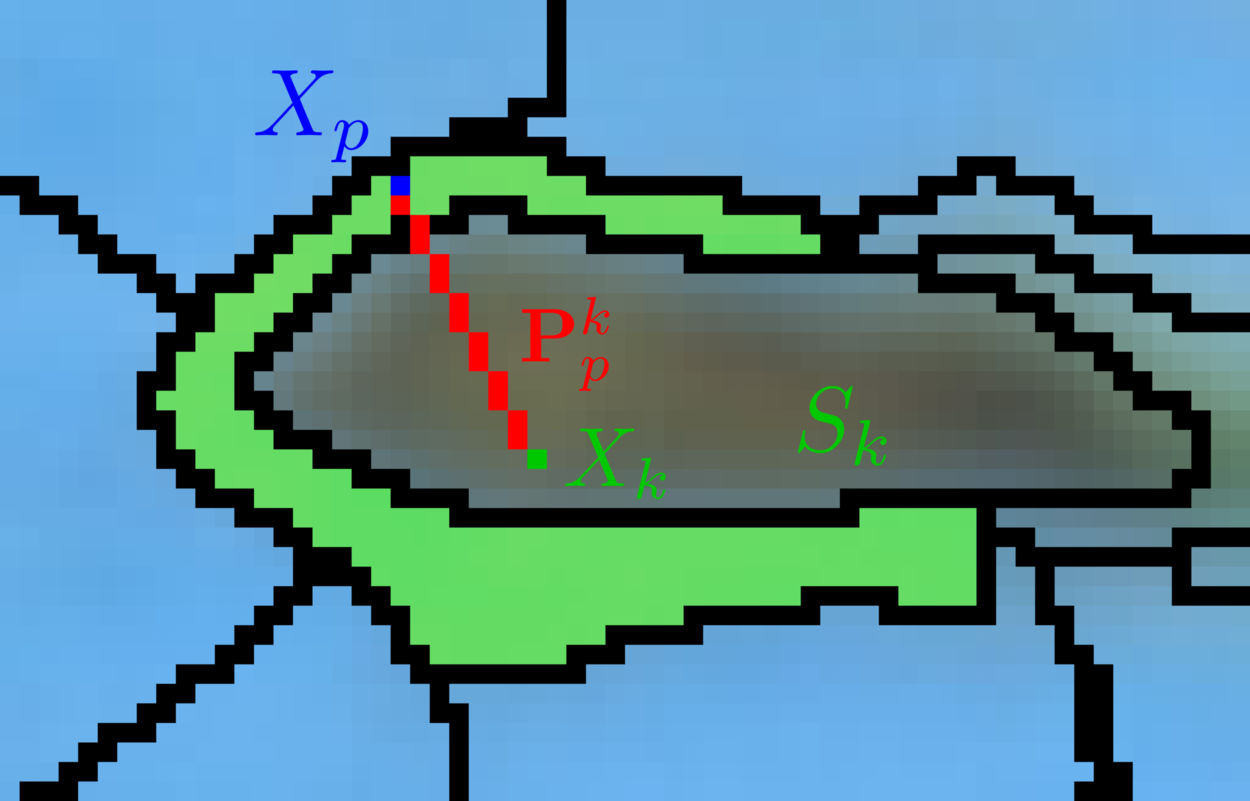}&
\includegraphics[height=0.075\textwidth,width=0.115\textwidth]{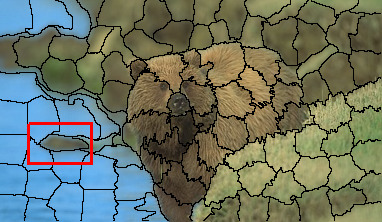}&
\includegraphics[height=0.075\textwidth,width=0.115\textwidth]{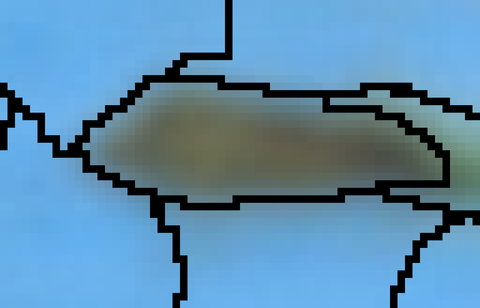}\\
\includegraphics[height=0.075\textwidth,width=0.115\textwidth]{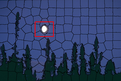}&
\includegraphics[height=0.075\textwidth,width=0.115\textwidth]{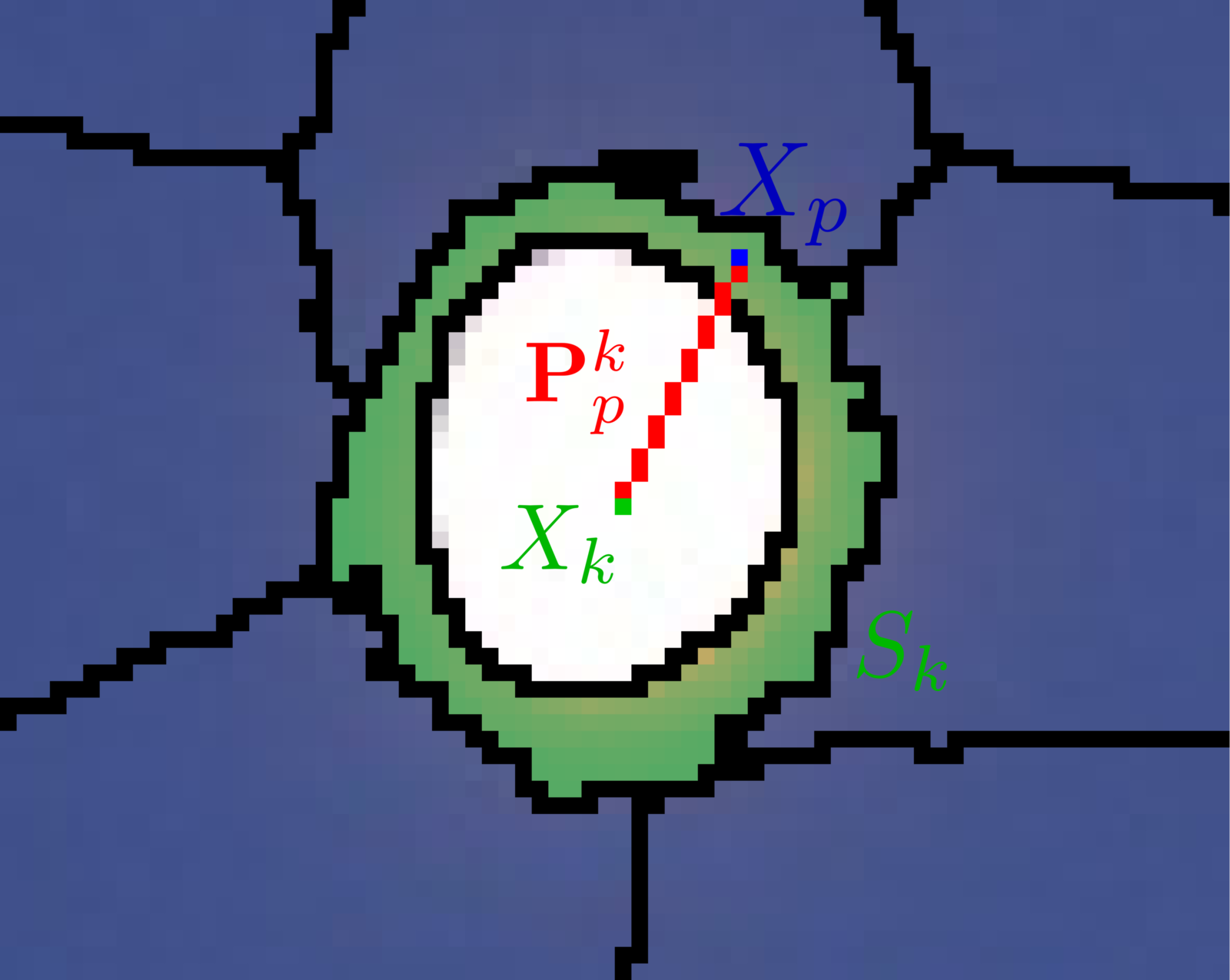}&
\includegraphics[height=0.075\textwidth,width=0.115\textwidth]{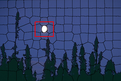}&
\includegraphics[height=0.075\textwidth,width=0.115\textwidth]{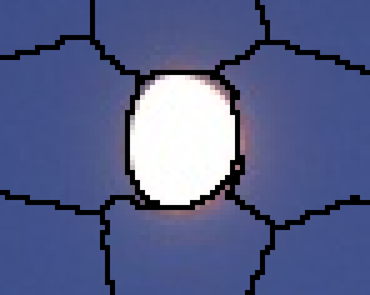}\\
\multicolumn{2}{c}{(a) SLIC irregular shapes}&\multicolumn{2}{c}{(b) SCALP regular shapes}\\
\end{tabular}
}
\caption{
Examples of irregular shapes obtained with SLIC \citep{achanta2012} (a)
and regular shapes obtained with SCALP using the color distance on the linear path \eqref{newdist0} (b).
With non regular shapes, the barycenter may fall outside the superpixel, 
and the linear path cross regions with different colors, penalizing the 
clustering distance.
} 
\label{fig:banana}
\end{figure}

\subsubsection{Adherence to Contour Prior}

Since the optimal color homogeneity may be not in line with the respect of image objects,
or fail to catch thin edges,
we propose to consider the information of
a contour prior map $\CC$ on the linear path.
Such map sets $\CC(p)$ to $1$ if a contour is detected 
at pixel $p$, and to $0$ otherwise. 
We propose a fast and efficient way to integrate a contour prior
by weighting the distance between a pixel and a superpixel cluster  
by $d_{\CC}(\PP^k_p)$, 
considering the maximum of contour intensity on $\PP^k_p$:
\begin{equation}
\label{newweight}d_{\CC}(\PP^k_p) = 1 + \gamma \hspace{0.1cm} \underset{q\in \PP^k _p}{\text{max}}\hspace{0.05cm}\CC(q) ,
\end{equation}
\noindent  with $\gamma\geq0$. 
Figure \ref{fig:max_contour} illustrates the selection of maximum
contour intensity on the linear path. 
When a high contour intensity is found on the path between a pixel $p$ and the barycenter of $S_k$, 
such term prevents this pixel 
to be associated to the superpixel, 
and all superpixel boundaries will follow more accurately the image contours.
The proposed framework can consider either soft contour maps, \emph{i.e.}, 
maps having values between $0$ and $1$, or binary maps.
It also adapts well to thick contour prior since only the maximum intensity on the path is considered. \\

Finally, we multiply this term to the color and spatial distances to ensure the respect of the images contours, and
the proposed distance $D$ to minimize during the
decomposition is defined as: 
\begin{equation}
   D(p,C_k)=\left(d_c(p,C_k,\PP^k_p) + d_s(p,C_k)\frac{m^2}{r^2}\right)d_{\CC}(\PP^k_p)  ,  \label{newdist} 
\end{equation}
\noindent 
with the spatial distance $d_s$ computed as Eq. \eqref{dist_slic}.
The SCALP method is summarized in Algorithm \ref{SCALP}. \\

\newcommand{\ttth}{0.14\textwidth}
\newcommand{\ttt}{0.175\textwidth}
\newcommand{\tttth}{0.17\textwidth}
\newcommand{\tttt}{0.15\textwidth}
 \begin{figure}[t!]
\centering
{\footnotesize
  \begin{tabular}{@{\hspace{1mm}}c@{\hspace{1mm}}c@{\hspace{1mm}}}
 \includegraphics[width=\ttt]{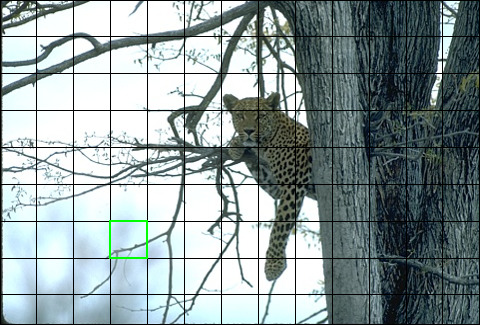}&
   \includegraphics[width=\ttt]{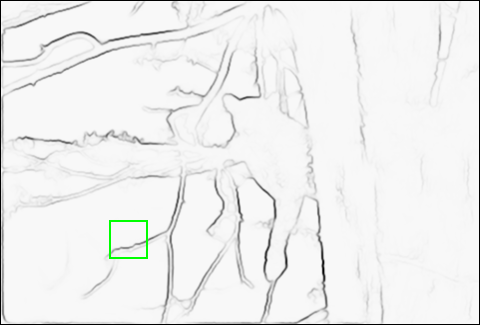}\\ \vspace{0.1cm}
   (a) Initial grid decomposition&(b) Contour prior\\ 
   \includegraphics[width=\tttt]{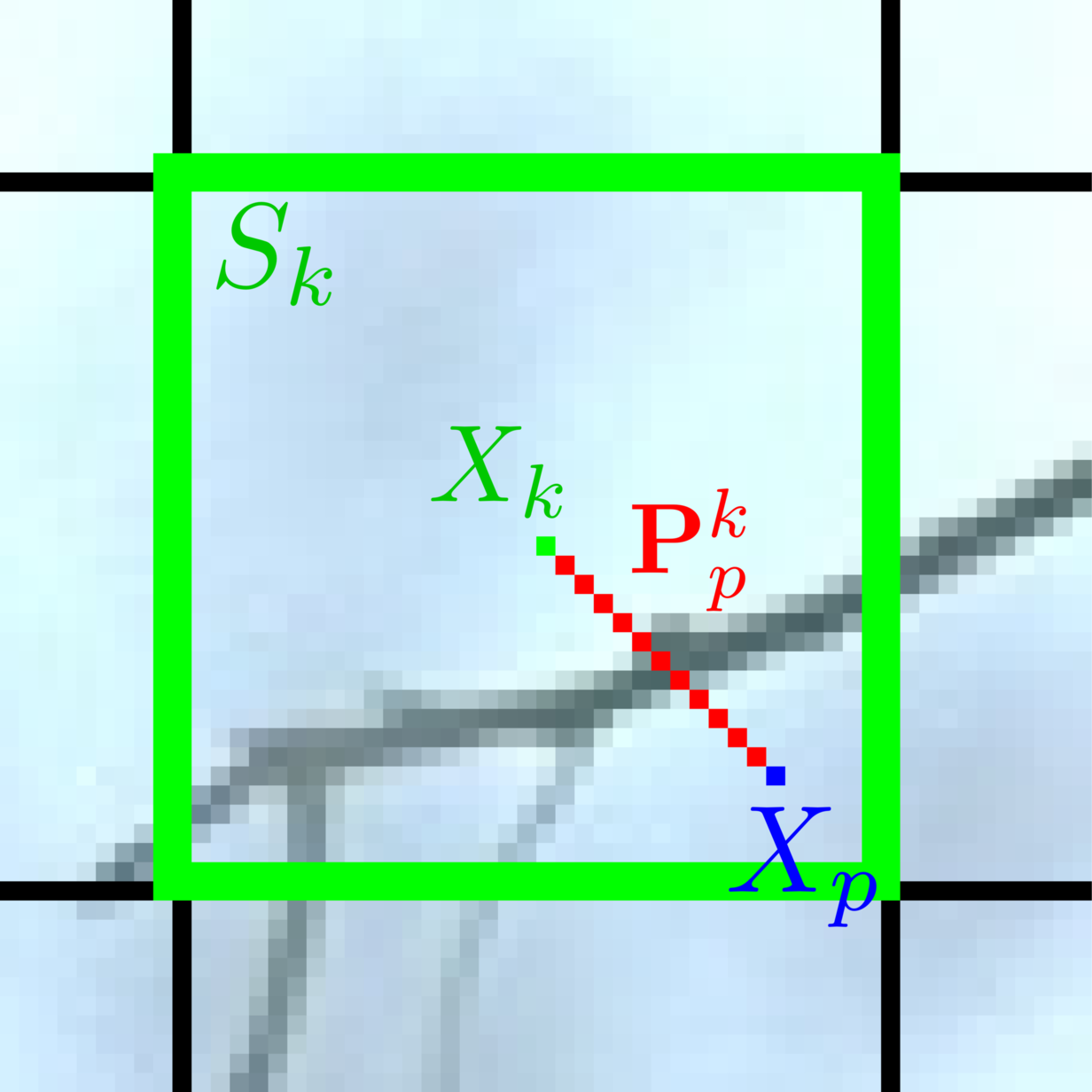}&
  \includegraphics[width=\tttt]{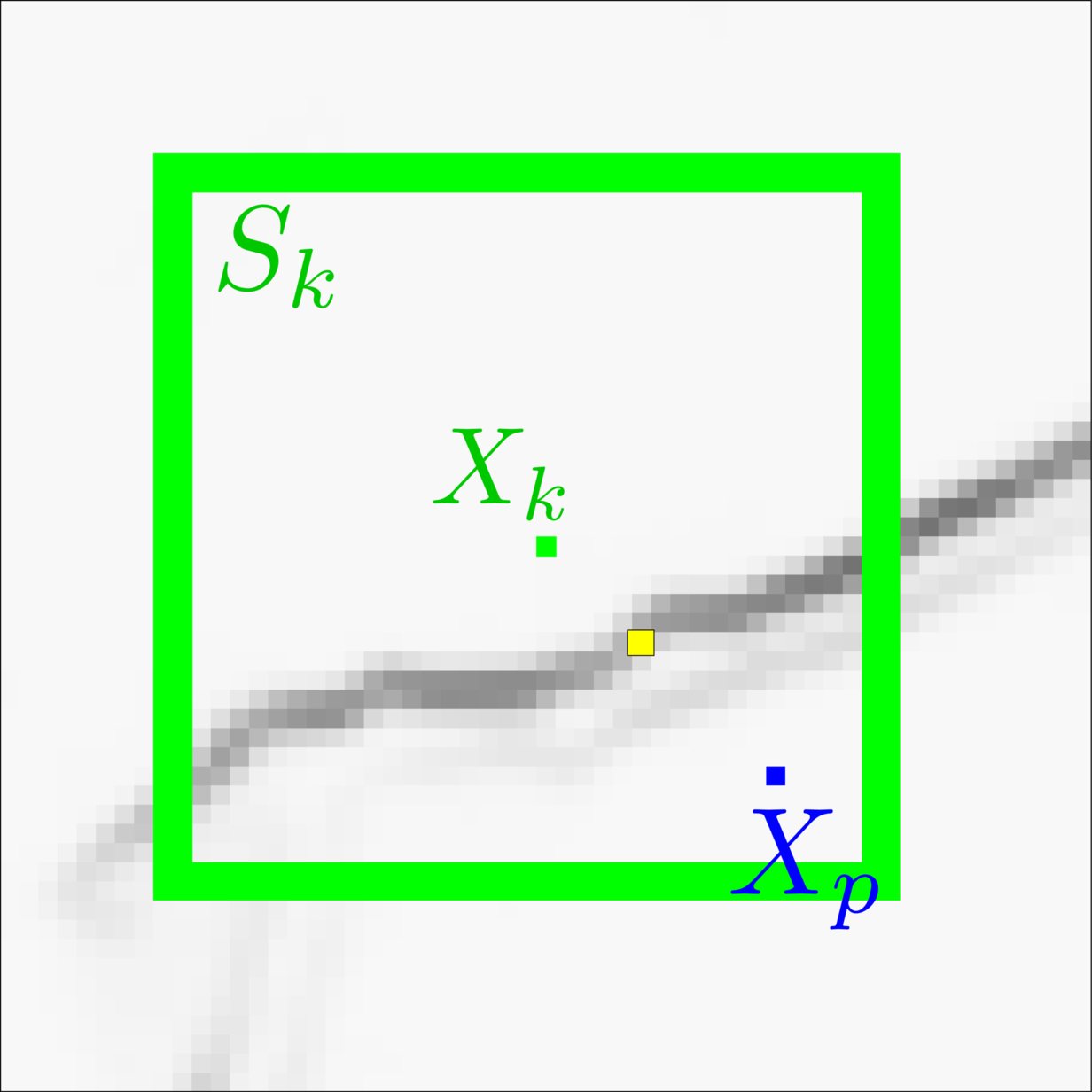}\\
 (c) Linear path $\PP^k_p$&(d) Maximum contour on $\PP^k_p$\\
  \end{tabular}
  }
  \caption{
  Illustration of SCALP first iteration starting from an initial grid (a) and using a contour prior (b).
  The linear path $\PP_p^k$ is defined for a pixel $p$ and a superpixel $S_k$ (c), and the
  maximum contour intensity (yellow pixel) is considered to prevent the crossing of image structures (d).
  }
  \label{fig:max_contour}
\end{figure}

\begin{algorithm}[ht!]
\footnotesize
  \caption{SCALP($I,K,\mathcal{C}$)}\label{SCALP}
\begin{algorithmic}[1]
\State Initialization of clusters $C_k \gets [F_k,X_k]$ from a regular grid
\State Initialization of superpixel labels $\SSS \gets 0$
\State Pre-computation of features ${\mathcal{F}_p}^{(2)}$ and ${\mathcal{F}_p}^{(1)}$  \eqref{speed_up}
\For{\emph{each iteration}}
  \State Distance $d \gets \infty$
  \For{ \emph{each $C_k$}} 
    \For{ \emph{each $p$ in a $(2r+1){\times}(2r+1)$ pixels window centered on $X_k$} }   
      \State Compute the linear path $\PP_p^k$ \citep{bresenham1965}
      \State Compute $D(p,C_k)$ using $\mathcal{C}$ and $\PP_p^k$ with \eqref{newdist} 
      \If{$D(p,C_k) < d(p)$}  
       \State $d(p) \gets D(p,C_k)$ 
        \State $\SSS(p) \gets k$
      \EndIf
    \EndFor
  \EndFor 
    \For{ \emph{each $C_k$}} 
 \State Update $[F_k,X_k]$ 
    \EndFor 
  \EndFor \vspace{-0.1cm}
  \State \textbf{return} $\mathcal{S}$
\end{algorithmic}
\end{algorithm}

\subsection{Initialization Constraint from Contour Prior}

In this section, we propose a framework 
to use an initial segmentation 
computed from a contour prior completion to constrain the superpixel decomposition.
To generate an image segmentation into regions from a contour map requires  additional steps but may help to improve the decomposition accuracy.
As stated in the introduction, 
although methods such as \citet{arbelaez2008,arbelaez2009} enable to segment an image into partitions considering a contour map,
they do not allow to control the size, the shape and the number of the produced regions.
We here propose a framework that uses an initial segmentation 
and produces a regular superpixel decomposition
within pre-segmented regions,
with control on the number of elements.
This way, we take advantage of the initial segmentation accuracy  
while providing an image decomposition into superpixels of regular sizes and shapes.
By initializing the decomposition within the computed regions, the initial superpixels better fit to the image content.
For instance, small regions can be initially segmented into one or several superpixels,
while they may fall between two initial superpixel barycenters, 
and would not be accurately segmented during the decomposition process.

{\color{black}
\subsubsection{Hierarchical Segmentation from Contour Detection}

In order to adapt an initial segmentation
to produce regular superpixels,
we propose to use a hierarchical segmentation, 
that can be computed from a contour map 
with methods such as \citet{arbelaez2008,arbelaez2009}.

Let $\mathcal{U}$ be a hierarchical segmentation that defines a contour probability map.
For any threshold, $\mathcal{U}$
produces a set of closed curves.
Regions segmented with low probability, \emph{i.e.}, with low intensity contours in $\mathcal{U}$
can be deleted with a thresholding step.
The thresholded closed contour map is denoted $\mathcal{U}_{\tau}$, 
for a threshold $\tau$, and
its corresponding decomposition into regions is denoted $\mathcal{R}_\tau=\{R_i\}$.
Figure \ref{fig:ucm_illustration}, illustrates 
the result obtained from a hierarchical segmentation for several thresholds.

\begin{figure}[h!]
\centering
{\small
\begin{tabular}{@{\hspace{0mm}}c@{\hspace{2mm}}c@{\hspace{0mm}}}
\includegraphics[width=0.225\textwidth]{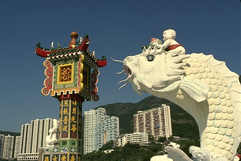}&
\includegraphics[width=0.225\textwidth]{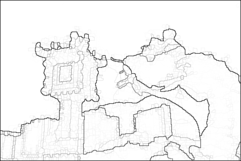}\\
Image &  $\tau=0$ \\
\includegraphics[width=0.225\textwidth]{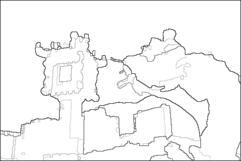}&
\includegraphics[width=0.225\textwidth]{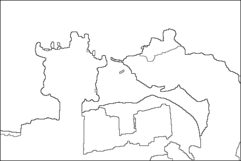}\\
$\tau=0.2$ &  $\tau=0.6$\\
\end{tabular}
} 
\caption{
{\color{black}Example of hierarchical
segmentation computed with  \citet{arbelaez2009} from a contour map obtained with \citet{dollar2013}.
The hierarchical segmentation is illustrated for several values of the threshold parameter $\tau$.}
}  
\label{fig:ucm_illustration}
\end{figure}
}

\subsubsection{Regular Decomposition into Superpixels from a Hierarchical Segmentation}

Once the hierarchical segmentation is obtained and thresholded, 
a merging step can be performed to remove the smallest areas. 
Such small regions should be merged to an adjacent one to respect the size regularity of the decomposition.
With $K$ the number of superpixels and $|I|$ the number of pixels of an image $I$,
the superpixel average size is  $s=|I|/K$.
A threshold $t\in[0,1]$ is set to merge regions containing less pixels than $s{\times}t$.
The segmentation probability of a region $R_i$ 
is $\underset{p\in \mathcal{B}(R_i)}{\min}\mathcal{U}_\tau(p)$, \emph{i.e.},
the lowest intensity among its boundary pixels $p\in \mathcal{B}(R_i)$.
The region $R_i$ is hence merged to its adjacent region $R_j$ that shares the 
boundary with the lowest segmentation probability:
\begin{equation}
\text{if} \hspace{0.3cm} |R_i| < s{\times}t ,\hspace{0.5cm}
\mathcal{R}_\tau(R_i) = \underset{j, p\in \mathcal{B}(R_i) \cap \mathcal{B}(R_j)}{\argmin} \mathcal{U}_\tau(p) . \label{spar}
\end{equation} 
These steps are illustrated in Figure \ref{fig:merging_ucm},
where the thresholding removes areas segmented with low probability 
and the merging prevents the segmentation of small regions.

\begin{figure*}[t!]
\centering
{\scriptsize
\begin{tabular}{@{\hspace{0mm}}c@{\hspace{1mm}}c@{\hspace{1mm}}c@{\hspace{1mm}}c@{\hspace{1mm}}c@{\hspace{0mm}}}
\includegraphics[height=0.16\textwidth,width=0.19\textwidth]{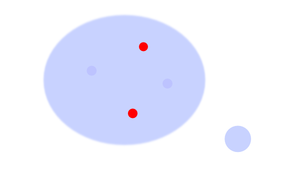}&
\includegraphics[height=0.16\textwidth,width=0.19\textwidth]{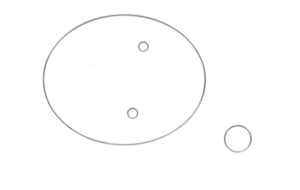}&
\includegraphics[height=0.16\textwidth,width=0.19\textwidth]{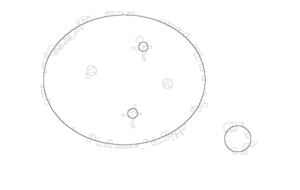}&
\includegraphics[height=0.16\textwidth,width=0.19\textwidth]{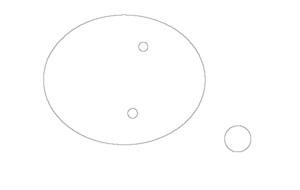}&
\includegraphics[height=0.16\textwidth,width=0.19\textwidth]{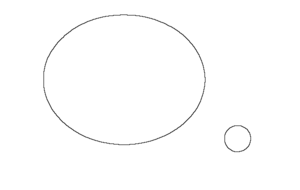}\\
(a) Image&(b) Contour map&(c) Hierarchical segmentation&(d) Thresholding&(e) Merging\\
\end{tabular}
} 
\caption{
Illustration of the thresholding and merging steps of the hierarchical segmentation (c)
computed from the contour map (b) of an image (a).
The thresholding step (d) enables to remove the areas segmented with low probability, 
\emph{i.e.}, the small blue circles and the segmentation artifacts.
Then, according to the condition in Eq. \eqref{spar}, smallest regions are removed (e), \emph{i.e.},
the red circles, although they have higher segmentation probability than the blue ones.
} 
\label{fig:merging_ucm}
\end{figure*}

A partition step then adds initial superpixels in the remaining regions.
If the resulting number of regions is lower than the number of superpixels $K$, 
superpixels are added according to the region size $|R_i|$.
In a region $R_i$, $\lfloor|R_i|/s\rfloor$ sub-regions are initialized by a spatial K-means approach \citep{lloyd82}, 
regardless of the color information.

The proposed approach thus adapts well to the superpixel size, and is
not sensitive to threshold settings.
The framework using the contour prior as a hard constraint
is illustrated in Figure \ref{fig:HC},
and will be denoted SCALP+HC in the following.
Note that although we here consider the segmentation as a hard constraint
to enforce the respect of image objects,
the image partition can be used to only initialize the superpixel repartition,
instead of using a regular grid.

\begin{figure*}[t!]
\centering
 \includegraphics[width=0.84\textwidth]{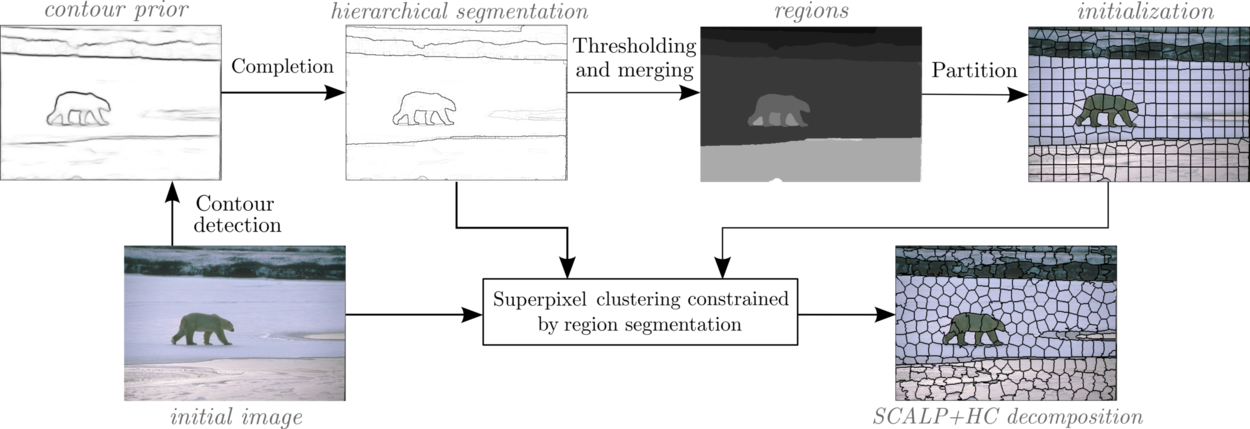}
 \caption{
 SCALP+HC framework using the contour prior as a hard constraint to provide an initial segmentation.
 A completion step produces a hierarchical segmentation from the contour map.
 Regions segmented with low probability are removed by a thresholding step, 
 and too small regions compared to the given superpixel size are merged to adjacent regions. 
 These regions can then be partitioned to provide a superpixel initialization.
 SCALP is independently performed in each region, taking advantage of the contour map accuracy while
 producing a regular decomposition that adapts well to local image content.
 }
 \label{fig:HC}
\end{figure*}

\section{Results}

\subsection{\label{valid}Validation Framework}

\subsubsection{Dataset}

We use the standard Berkeley segmentation dataset (BSD) \citep{martin2001} 
to evaluate our method and compare to state-of-the-art ones. 
This dataset contains 200 various test images of size $321{\times}481$ pixels.
At least 5 human ground truth decompositions are provided per image
to compute evaluation metrics in terms of
consistency to the image objects, and
contour adherence.

\subsubsection{Metrics}

To evaluate our method and compare to other state-of-the-art frameworks, 
we use standard superpixel evaluation metrics.
The achievable segmentation accuracy (ASA)
measures the consistency of the decomposition to the image objects.
Boundary recall (BR) and contour density (CD) are used to measure
the detection accuracy according to the ground truth image contours.
We also propose to evaluate the contour detection performance of the superpixel methods
by computing the precision-recall (PR) curves \citep{martin2004}.  
Finally, we report the shape regularity criteria (SRC) \citep{giraud2017_src} 
that measures the regularity of the produced superpixels.

For each image $I$ of the dataset, human ground truth segmentations are provided.
The reported results are averaged on all segmentations.
A ground truth decomposition 
is denoted $\TT=\{T_i\}_{i\in \{1,\dots ,|\TT|\}}$, with $T_i$ a segmented region,
and we consider a superpixel segmentation $\SSS=\{S_k\}_{k\in \{1,\dots,|\SSS|\}}$.

\subsubsection*{Respect of image objects} 

For each superpixel $S_k$ of the decomposition result, 
the largest possible overlap with a ground truth region $T_i$ can be computed with
ASA, which computes the average overlap percentage for all superpixels: 
\begin{equation}
\text{ASA}(\SSS,\TT) = \frac{1}{|I|}\sum_{k}\max_i|S_k\cap T_i|.  \label{asa}
\end{equation}

Note that recent works, \emph{e.g.}, \citet{giraud_jei_2017,stutz2016} show
the high correlation between the undersegmentation error \citep{neubert2012} 
and the ASA metric \eqref{asa}.
Therefore, the ASA measure is sufficient to evaluate 
the respect of image objects.

\subsubsection*{Contour Detection} 
The BR metric measures the detection of ground truth contours $\mathcal{B(\TT)}$ 
by the computed superpixel boundaries $\mathcal{B}(\SSS)$.
If a ground truth contour pixel has a decomposition contour pixel
at an $\epsilon$-pixel distance,
it is considered as detected, and
BR is defined as the percentage of detected ground truth contours: 
\begin{equation}
\text{BR}(\SSS,\TT) = \frac{1}{|\mathcal{B}(\TT)|}\sum_{p\in\mathcal{B}(\TT)}\delta[\min_{q\in\mathcal{B}(\SSS)}\|p-q\|< \epsilon]  ,   \label{br}
\end{equation}
\noindent with 
$\delta[a]=1$ when $a$ is true and $0$ otherwise,
and $\epsilon$ set to $2$ as in, \emph{e.g.}, \citet{vandenbergh2012}. 
However, this measure only considers true positive detection,
and does not consider the number of produced superpixel contours.
Therefore, methods that produce very irregular superpixels are likely to 
have high BR results.
To overcome this limitation, as in  \citet{machairas2015,zhang2016},
the contour density (CD)
can be considered to penalize a large number of superpixel boundaries $\mathcal{B}(\SSS)$.
In the following, we report CD over BR results, with CD defined as:
\begin{equation}
\text{CD}(\SSS) = \frac{|\mathcal{B}(\SSS)|}{|I|}  .  \label{cd}
\end{equation}
When considering decompositions with the same CD, \emph{i.e.}, 
the same number of superpixel boundaries, 
BR results can be relevantly compared.
Higher BR with the same CD indicates that the produced superpixels better detect image
contours.

The PR framework \citep{martin2004}
enables to measure the contour detection performances. 
PR curves consider both boundary recall (BR) \eqref{br},  
\emph{i.e.},  true positive detection, or percentage of detected ground truth contours,
and precision $\text{P}=|\mathcal{B}(\SSS)\cap\mathcal{B}(\TT)|/|\mathcal{B}(\SSS)|$, \emph{i.e.}, 
percentage of accurate detection on produced superpixel boundaries.
They are computed from an input map, where the intensity in each pixel represents the confidence of being 
on an image boundary.
As in \citet{vandenbergh2012}, 
we consider the average of superpixel boundaries obtained at different scales,
ranging from 25 to 1000 superpixels, 
 to provide a contour detection.  
In the following, to summarize the contour detection performances, 
we report the maximum F-measure defined as: 
\begin{equation}
 \text{F} = \frac{2.\text{P}.\text{BR}}{\text{P}+\text{BR}} . \label{fmeasure}
\end{equation}

\subsubsection*{Shape Regularity} 
To evaluate the regularity of a decomposition in terms of superpixel shape, 
we use the shape regularity criteria (SRC) introduced in \citet{giraud2017_src},
and defined for a decomposition $\SSS$ as follows: 
\begin{equation}
\text{SRC}(\SSS) = \sum_k{\frac{|S_k|}{|I|}}.\frac{\text{CC}(H_{S_k})}{\text{CC}(S_k)}\text{V}_{\text{xy}}(S_k) , \label{circu}
\end{equation}
\noindent where 
$\text{V}_\text{xy}(S_k) = {\min(\sigma_x,\sigma_y)/\max(\sigma_x,\sigma_y)}$,
evaluates the balanced repartition of the shape $S_k$
with $\sigma_x$ and $\sigma_y$ 
the square root of standard deviations of pixel positions $x$ and $y$ in $S_k$,
 $H_{S_k}$ is the convex hull containing $S_k$, and
CC measures the ratio between the perimeter and the area of the considered shape.
The SRC measure has been proven to be more robust and accurate than the
circularity metric \citep{schick2012} used in several superpixel works.

\newcommand{\www}{0.24\textwidth}
\newcommand{\hhh}{0.20\textwidth}
\begin{figure*}[t]
\centering
{\footnotesize
\begin{tabular}{@{\hspace{0mm}}c@{\hspace{1mm}}c@{\hspace{1mm}}c@{\hspace{1mm}}c@{\hspace{0mm}}}
\vspace{0.1cm}
\includegraphics[height=\hhh,width=\www]{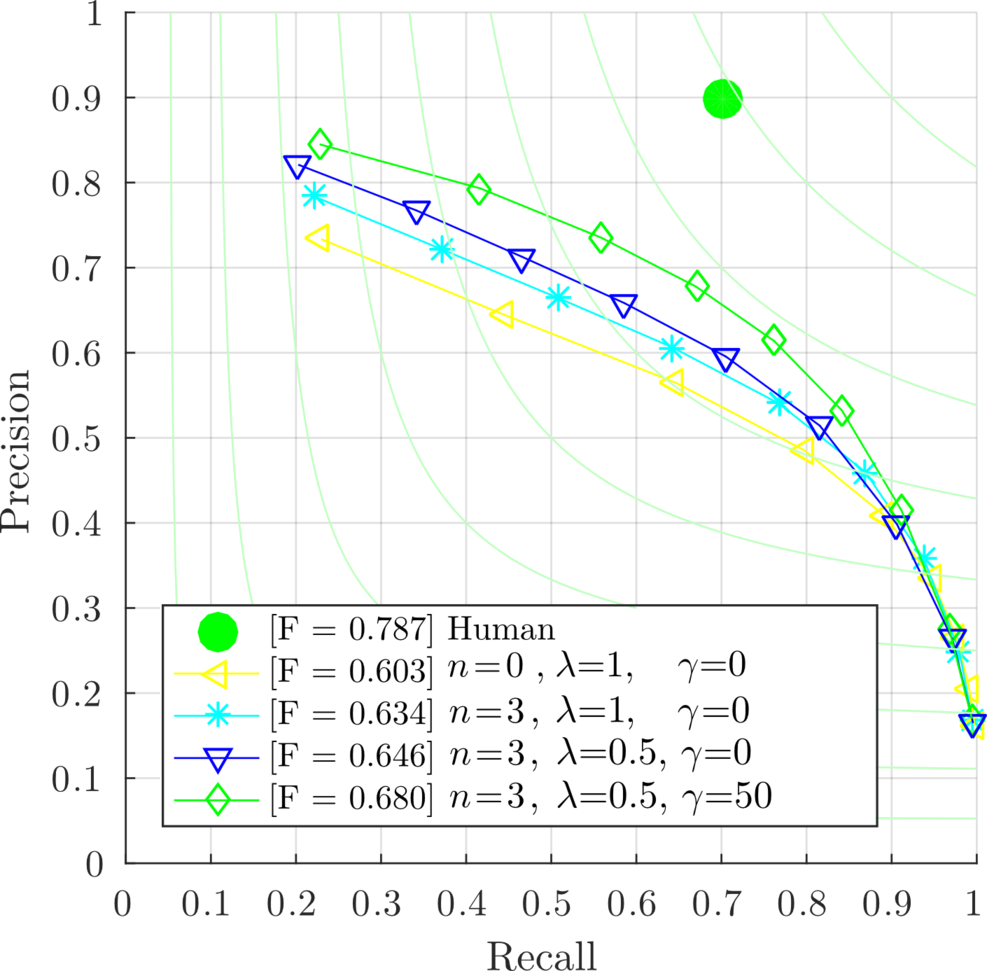}&
\includegraphics[height=\hhh,width=\www]{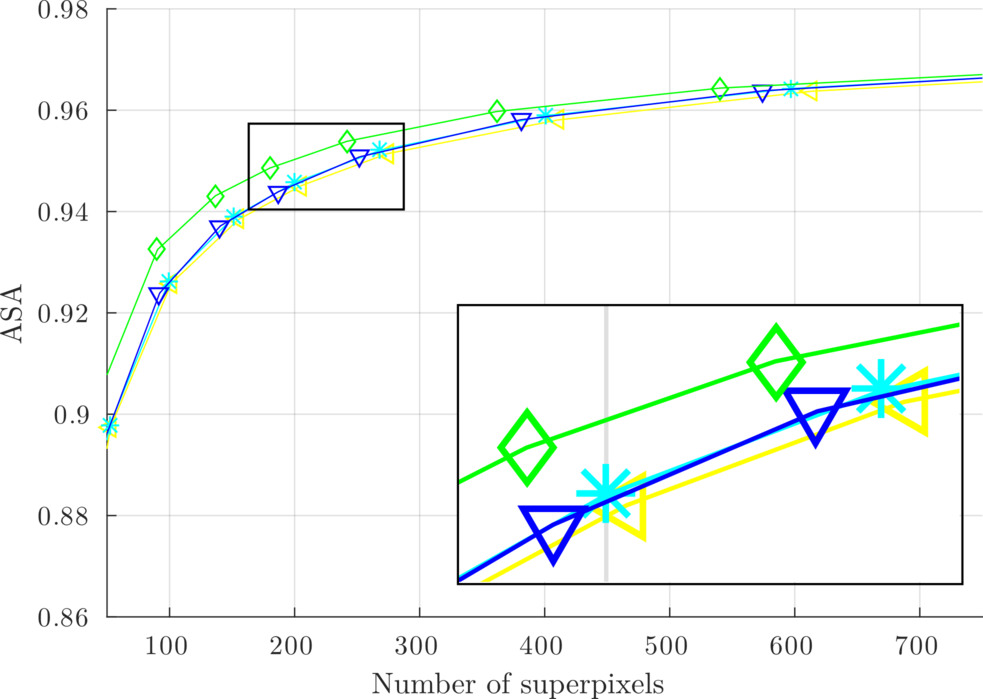}&
\includegraphics[height=\hhh,width=\www]{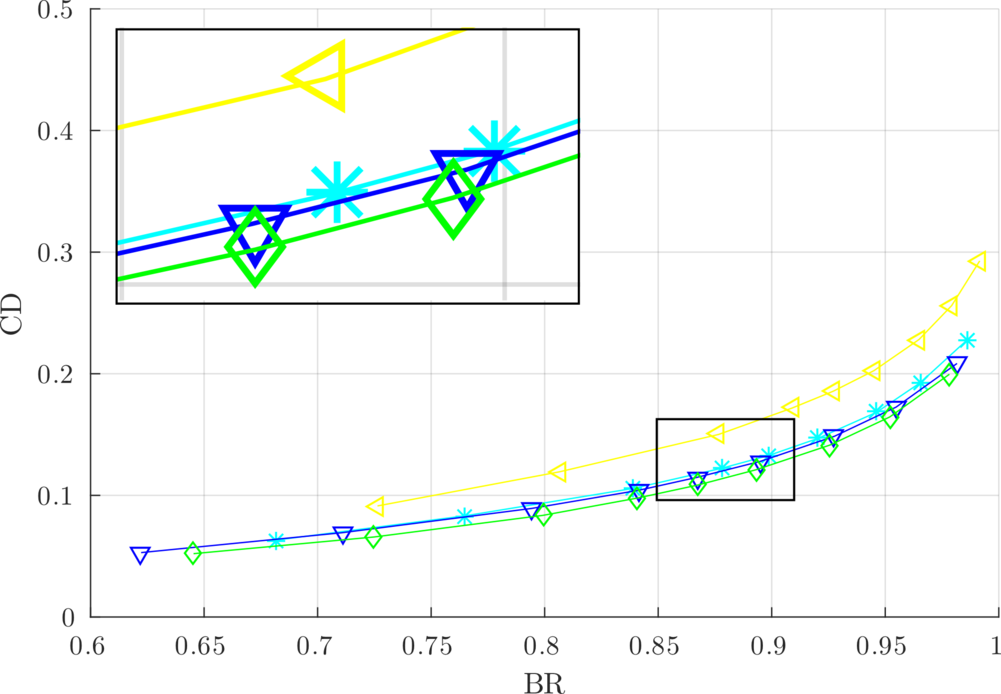}&
\includegraphics[height=\hhh,width=\www]{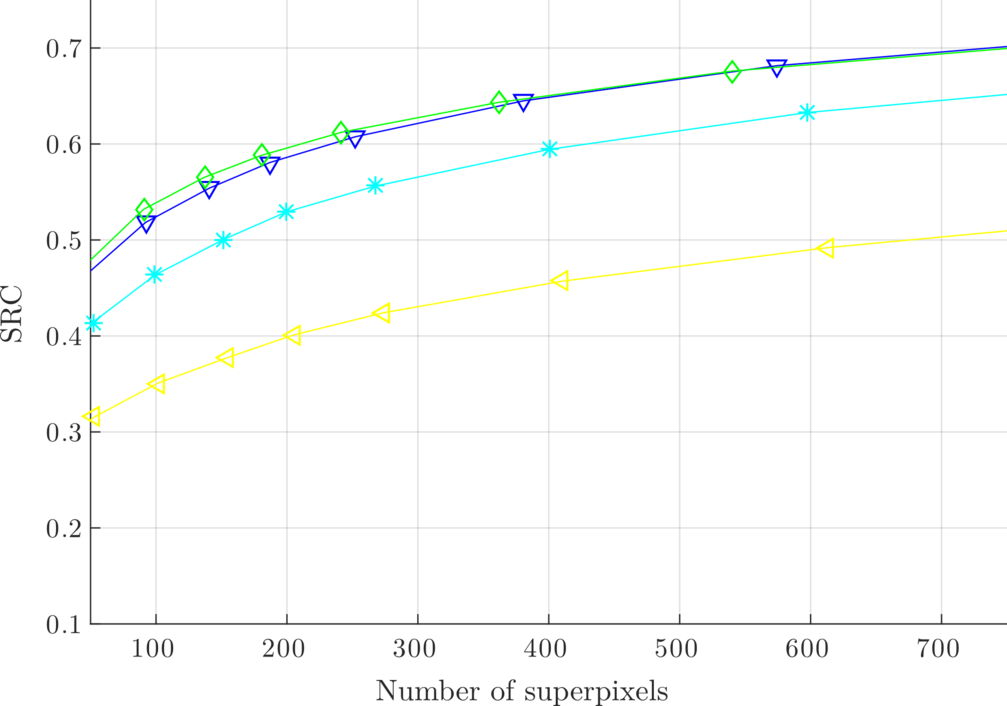}\\
\includegraphics[height=\hhh,width=\www]{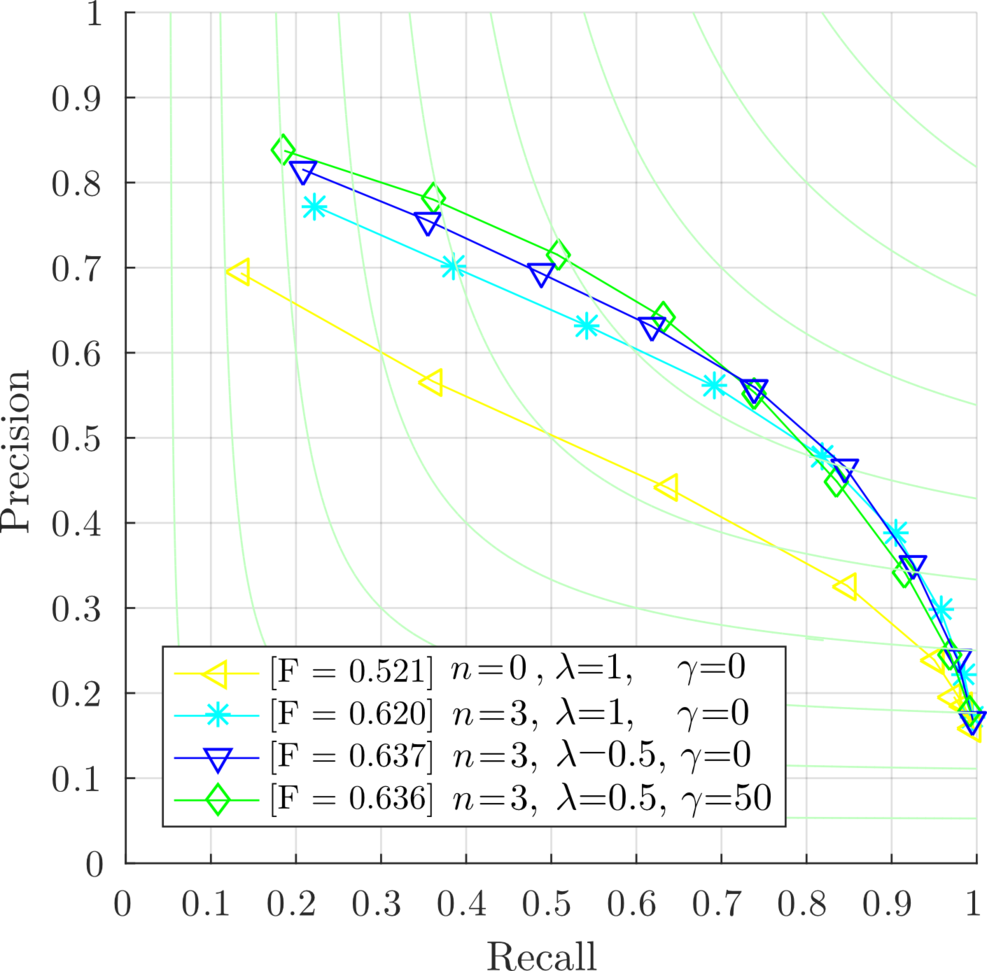}&
\includegraphics[height=\hhh,width=\www]{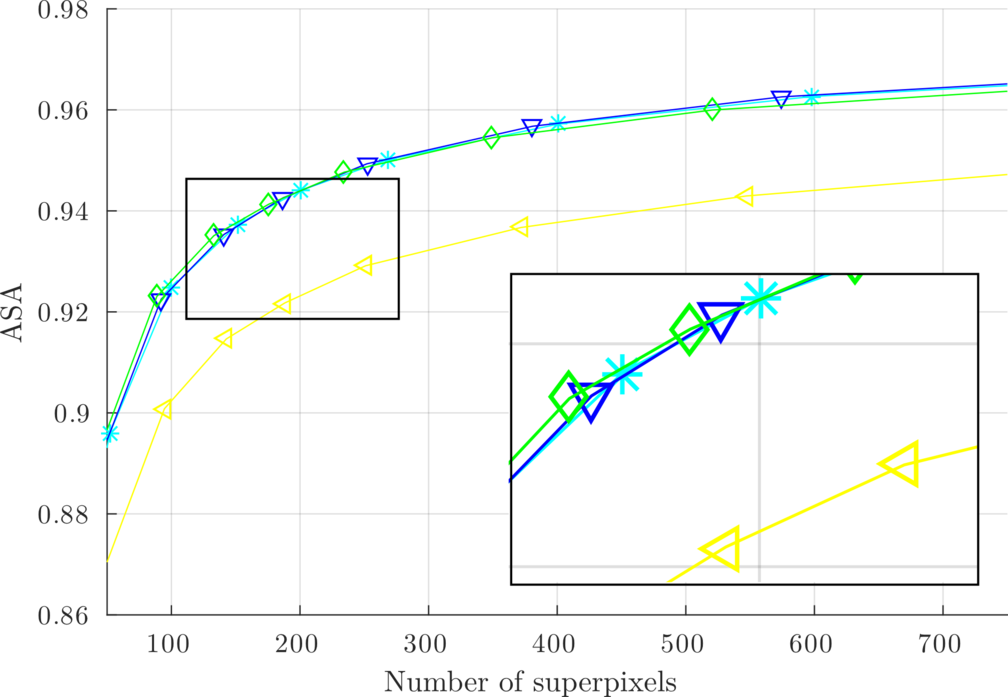}&
\includegraphics[height=\hhh,width=\www]{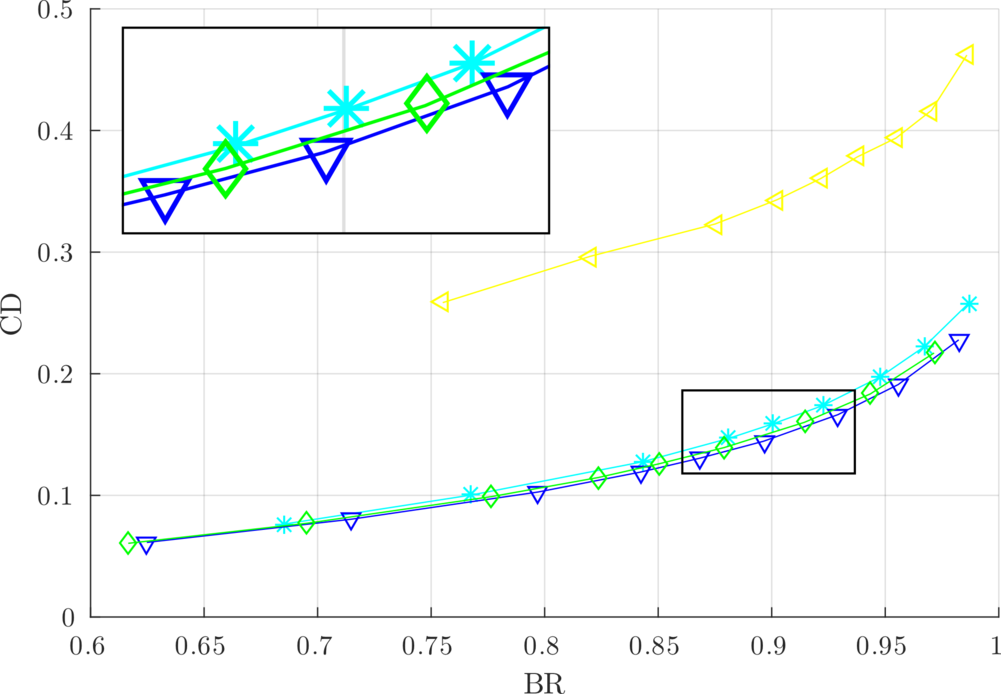}&
\includegraphics[height=\hhh,width=\www]{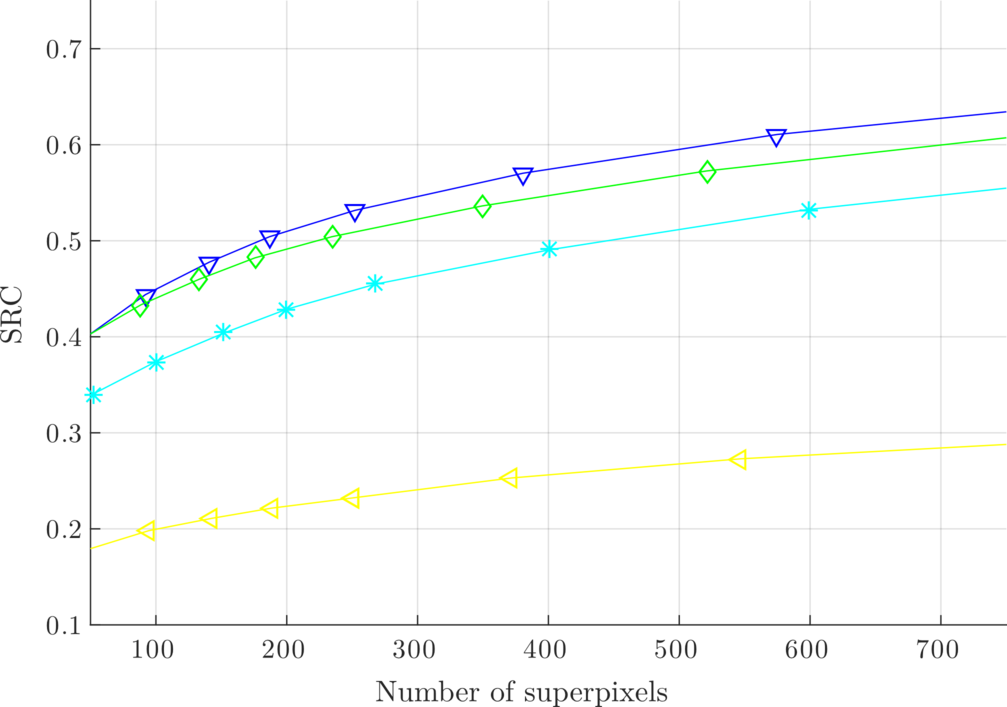}\\
\end{tabular}
}
\caption{Evaluation of the SCALP distance parameters on PR, ASA, CD over BR and SRC 
metrics on initial (top) and noisy images (bottom).
Each contribution increases the decomposition accuracy for both initial and noisy images.
The parameter $n$ in \eqref{patch} sets the use of the neighborhood in
the clustering distance, and $\lambda$ in \eqref{newdist0} and $\gamma$ in \eqref{newweight} 
respectively set the influence of the color distance and contour prior 
along the linear path.
With $n=0$ in
\eqref{patch}, $\lambda=1$ in \eqref{newdist0} and $\gamma=0$ in \eqref{newweight},
the framework is reduced to the method of \citet{chen2017}.
}
\label{fig:distance}
\end{figure*}

\subsubsection{\label{param_settings}Parameter Settings}

SCALP was implemented with MATLAB using single-threaded C-MEX code,
on a standard Linux computer.
We consider in $d_c$ and $d_s$ more advanced spectral features introduced in \citet{chen2017}.
They are designed in a high dimensional space ($6$ for color, and $4$ for spatial features).
The linear path 
between a pixel and the barycenter of a superpixel is computed 
with \citet{bresenham1965}.
In \eqref{patch}, the parameter $\sigma$ 
is empirically set to $40$ 
and  $\mathcal{P}(p)$ is defined as a $7{\times}7$ pixel neighborhood around a pixel $p$, so
$n=3$.
In the proposed color distance \eqref{newdist0}, $\lambda$ is set to $0.5$, and
$\gamma$ to $50$ in  \eqref{newweight}.
The compactness parameter $m^2$ is set to $0.075r^2$ in the final distance \eqref{newdist}, as in \citet{chen2017}.
This parameter offers a good trade-off between adherence to contour prior and compactness. 
The number of clustering iterations is set to $5$, contrary to \citet{chen2017} that uses $20$ iterations,
since SCALP converges faster.
Unless mentioned, when used, the contour prior is computed with \citet{dollar2013}. 
Finally, when using the contour prior as a hard constraint (SCALP+HC), 
we respectively set parameters $t$ and $\tau$ during the region fusion \eqref{spar} to $0.15$ and $0.4$,
and compute a hierarchical segmentation with \citet{arbelaez2009}.
In the following, when reporting results on noisy images, we use a
white additive Gaussian noise of variance 20.

\newcommand{\pppp}{0.155\textwidth}
\newcommand{\pppph}{0.125\textwidth}
\begin{figure}[t!]
\centering
{\scriptsize
\begin{tabular}{@{\hspace{0mm}}c@{\hspace{1mm}}c@{\hspace{1mm}}c@{\hspace{0mm}}}
Initial image& 
$n$=$0$, $\lambda$=$1$, $\gamma$=$0$ &
  {\color{blue}$n$=$3$}, $\lambda$=$1$, $\gamma$=$0$\\
 \includegraphics[width=\pppp]{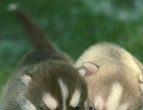}&
\includegraphics[width=\pppp]{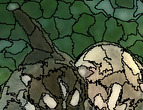}&
\includegraphics[width=\pppp]{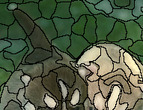}\\
\end{tabular}
\begin{tabular}{@{\hspace{0mm}}c@{\hspace{1mm}}c@{\hspace{0mm}}}
$n$=$3$, {\color{blue}$\lambda$=$0.5$}, $\gamma$=$0$&
$n$=$3$, $\lambda$=$0.5$, {\color{blue}$\gamma$=$50$} \\
\includegraphics[width=\pppp]{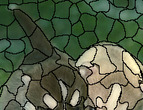}&
\includegraphics[width=\pppp]{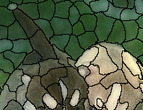}\\
\end{tabular}
} 
\caption{Visual impact of the distance parameters.
Each contribution progressively increases the decomposition accuracy by adding more relevant features.
} 
\label{fig:dist_img}
\end{figure}

\subsection{Influence of Parameters}
\label{ssec:influ_params}

\subsubsection{Distance Parameters}
\label{sssec:dist}

We first measure the influence of the distance parameters in \eqref{newdist} on SCALP performances. 
In Figure \ref{fig:distance}, 
we report results on PR, ASA, CD over BR and SRC curves for different distance settings, on both initial
and noisy BSD images.
First, we note that the neighboring pixels ($n=3$ in \eqref{patch}) increase the decomposition accuracy.
The color features ($\lambda=0.5$ in \eqref{newdist0})  
also improve the results, in terms of respect of image objects and regularity.
Finally, the contour prior ($\gamma=50$ in \eqref{newweight})
along the linear path enables to reach high contour detection (PR) and also increases the
performances on superpixel metrics.
On noisy images, the accuracy of the contour prior is degraded,
but it still provides higher ASA performances on respect of image objects.
Note that if $n=0$ in
\eqref{patch}, $\lambda=1$ in \eqref{newdist0} and $\gamma=0$ in \eqref{newweight},
the method is reduced to the framework of \citet{chen2017}.

Figure \ref{fig:dist_img} illustrates the decomposition result for these distance parameters on a BSD image.  
With only the features used in \citet{chen2017}, \emph{i.e.}, with $n=0$, $\lambda=1$, $\gamma=0$, 
the decomposition boundaries are very irregular.
The neighborhood information greatly reduces the noise at the superpixel boundaries.
The color distance on the linear path improves the superpixel regularity and provides more compact shapes.
Finally, the contour information enables to more efficiently catch the object structures and 
to respect the image contours.

\subsubsection{Contour Prior}

We also investigate the influence of the contour prior.
The computation of the contour information should not be sensitive to textures and high local image gradients, and many
efficient methods have been proposed in the literature (see for instance references in \citet{arbelaez2011}).
The performances of our method  
are correlated to the contour detection accuracy, but we demonstrate that improvements are obtained even with basic
contour detections.

A fast way to obtain such basic contour detection, 
which would be robust to textures and high gradients, 
is to average the boundaries of superpixel decompositions obtained at multiple scales.
We propose to consider the same set of scales $\mathcal{K}=\{K\}$ 
used for computing the PR curves.
All resulting superpixels boundaries $\mathcal{B}(\SSS^K)$ of a decomposition $\SSS^K$, 
computed at scale $K\in \mathcal{K}$ are averaged:
\begin{equation}
 \bar{\mathcal{B}} = \frac{1}{|\mathcal{K}|}\sum_{K\in \mathcal{K}}{\mathcal{B}(\mathcal{S}^K)}. \label{ms}
 \end{equation}
 The average $\bar{\mathcal{B}}$ can then be thresholded 
 to remove low confidence boundaries
and provide an accurate contour prior $\mathcal{C}$.
Figure \ref{fig:ms} illustrates the computation of the contour prior $\mathcal{C}$
from superpixel boundaries. 
Note that the decompositions at multiple scales $K$ are independent 
and can be computed in parallel.

In Figure \ref{fig:contour}, we provide results obtained by using different contour prior:
the contour detection from multiple scale decompositions, 
using \citet{achanta2012} with a threshold of the boundary map \eqref{ms} set to $0.5$, 
from the {globalized probability of boundary} algorithm \citep{maire2008},
a method using learned sparse codes of patch gradients \citep{xiaofeng2012}, 
and from a structured forests approach  \citep{dollar2013}.
The results on all metrics are improved with the accuracy of the provided contour detection.
Nevertheless, we note that even simple contour priors
enable to improve the superpixel decomposition adherence to boundaries.
In the following, reported results are computed using \citet{dollar2013}.

\newcommand{\ttthc}{0.11\textwidth}
\newcommand{\tttc}{0.155\textwidth}
 \begin{figure}[t!]
\centering
{\footnotesize
  \begin{tabular}{@{\hspace{0mm}}c@{\hspace{1mm}}c@{\hspace{1mm}}c@{\hspace{0mm}}}
 \includegraphics[width=\tttc]{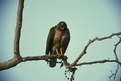}&
   \includegraphics[width=\tttc]{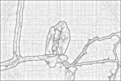}&
   \includegraphics[width=\tttc]{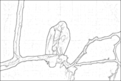}\\
   Initial image&Boundary average $\bar{\mathcal{B}}$& 
   Contour map $\mathcal{C}$\\
  \end{tabular}
  }
  \caption{Illustration of contour detection from superpixel boundaries
  computed with \cite{achanta2012}
   at multiple scales. Boundaries are averaged and thresholded to
  provide, in a fast and simple manner, an accurate contour prior. 
  The threshold of the boundary map \eqref{ms} is set to $0.5$.
  }
  \label{fig:ms}
\end{figure}

\newcommand{\hw}{0.46\textwidth}
\newcommand{\wh}{0.32\textwidth}
\newcommand{\hhw}{160pt}
\newcommand{\ww}{0.23\textwidth}
\newcommand{\hh}{100pt}
\newcommand{\hhi}{75pt}

\begin{figure}[t!] 
\centering
{\footnotesize
\begin{tabular}{@{\hspace{0mm}}c@{\hspace{1mm}}c@{\hspace{0mm}}}
\includegraphics[height=0.22\textwidth,width=0.23\textwidth]{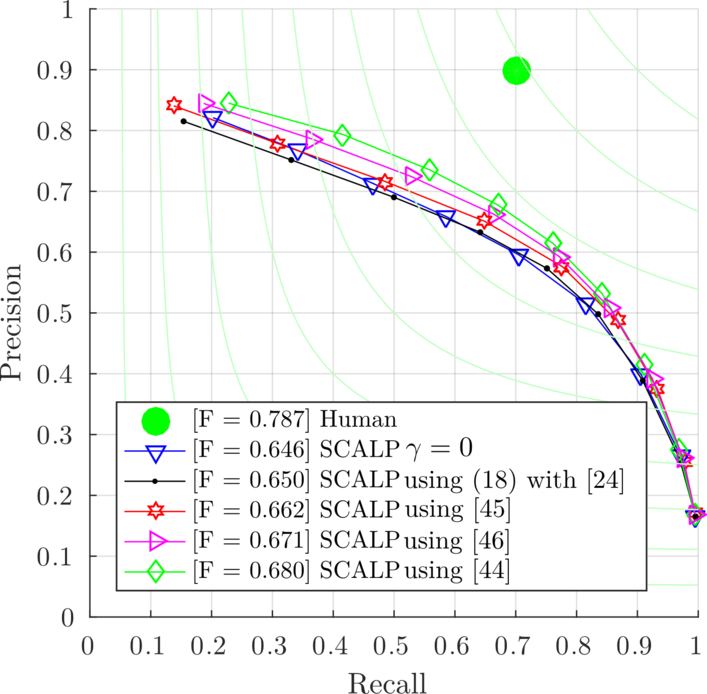}&
\includegraphics[height=0.22\textwidth,width=0.23\textwidth]{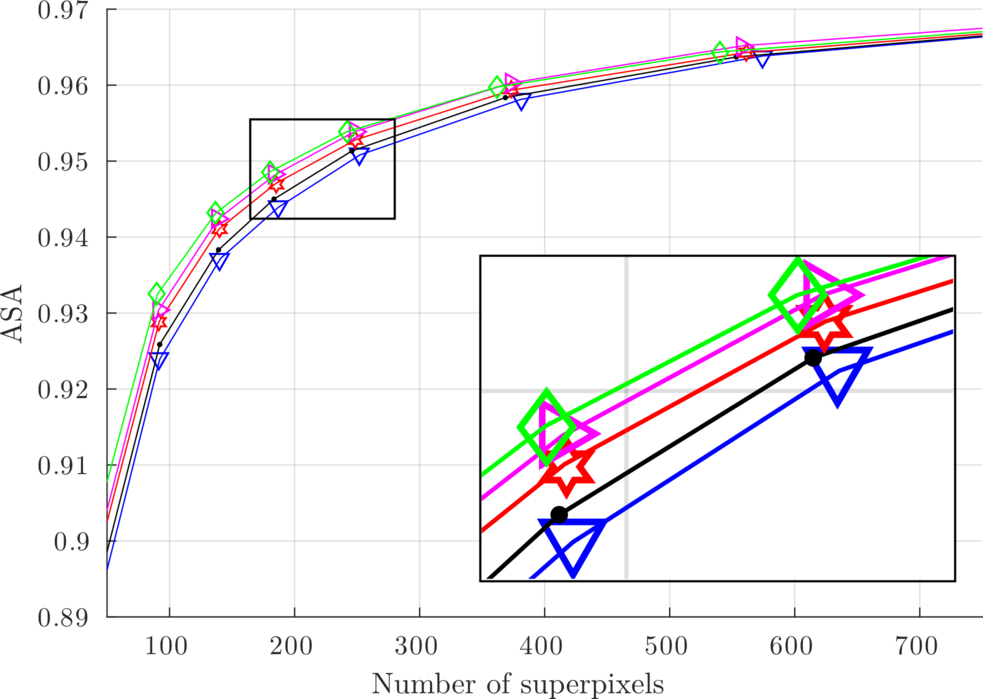}\\
\end{tabular}
}
\caption{
Evaluation of different contour priors.
Even a simple contour detection from averaged superpixel boundaries at multiple scales
improves the  adherence to image contours.
} 
\label{fig:contour}
\end{figure}

\newcommand{\lllh}{0.21\textwidth}
\newcommand{\llll}{0.24\textwidth}
\begin{figure*}[t!]
\centering
{\footnotesize
\begin{tabular}{@{\hspace{0mm}}c@{\hspace{1mm}}c@{\hspace{1mm}}c@{\hspace{1mm}}c@{\hspace{0mm}}}
\includegraphics[height=\hhh,width=\www]{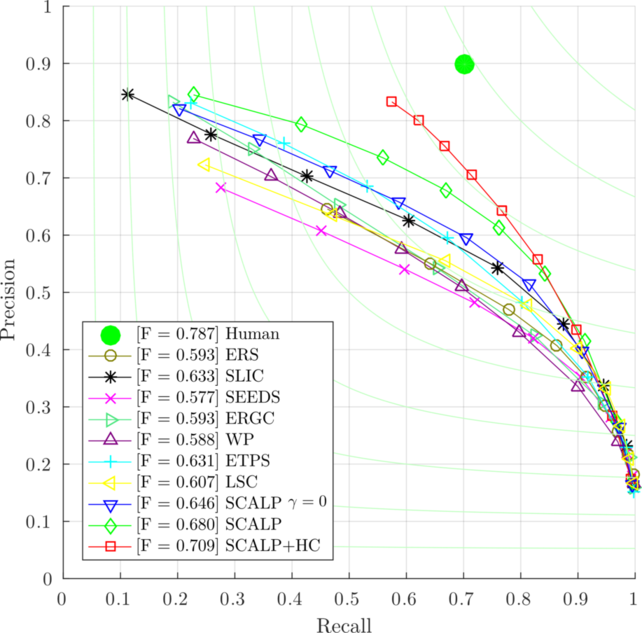}&
\includegraphics[height=\hhh,width=\www]{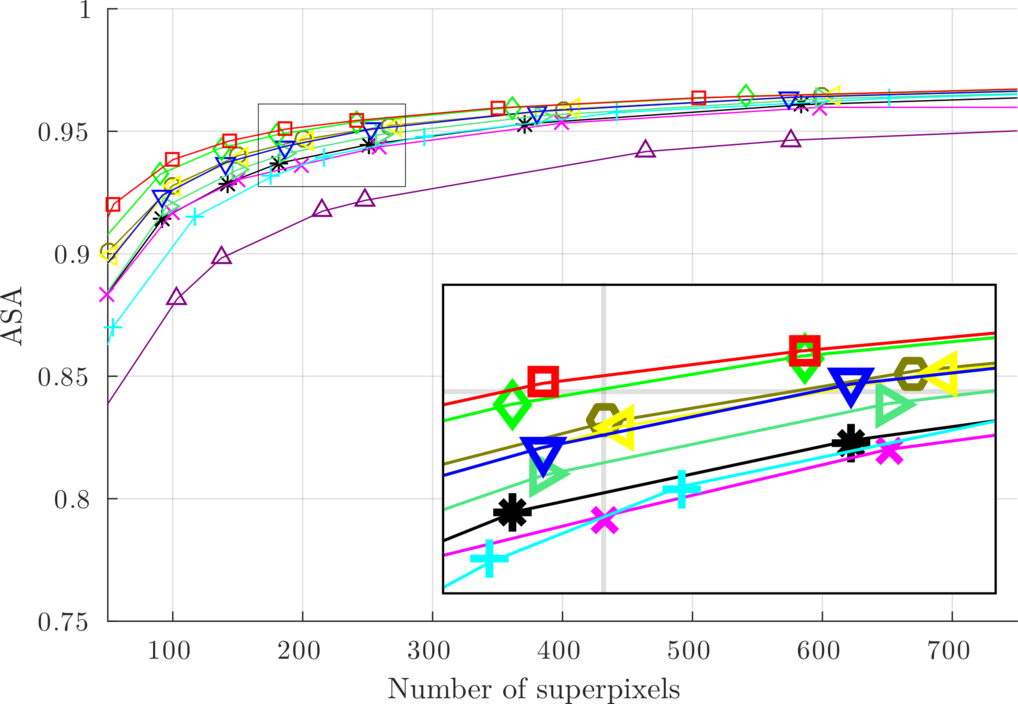}&
\includegraphics[height=\hhh,width=\www]{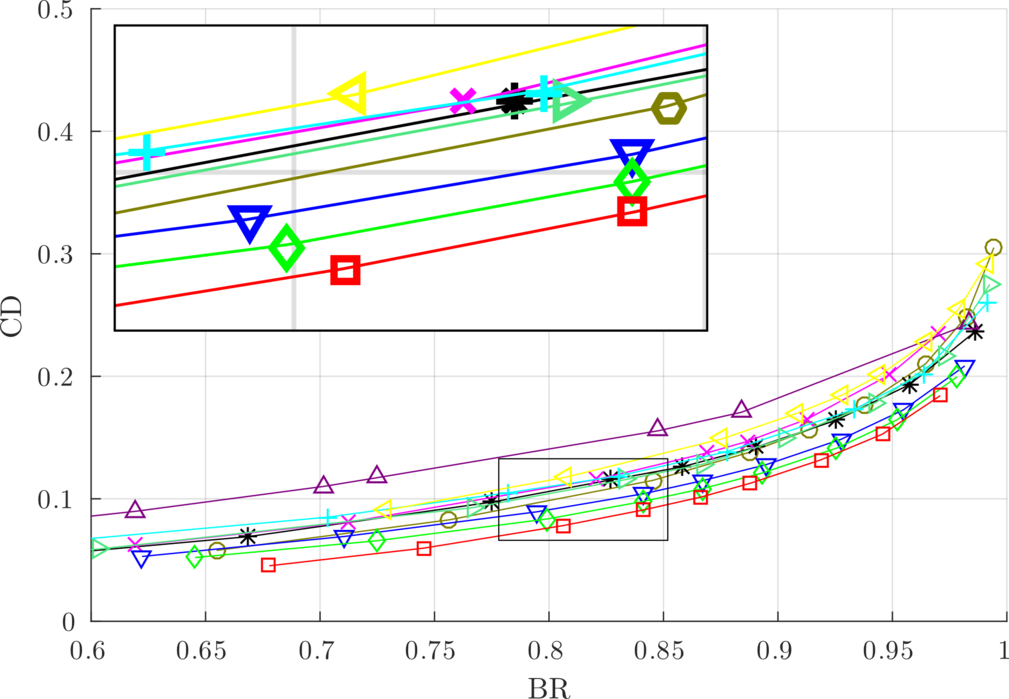}&
\includegraphics[height=\hhh,width=\www]{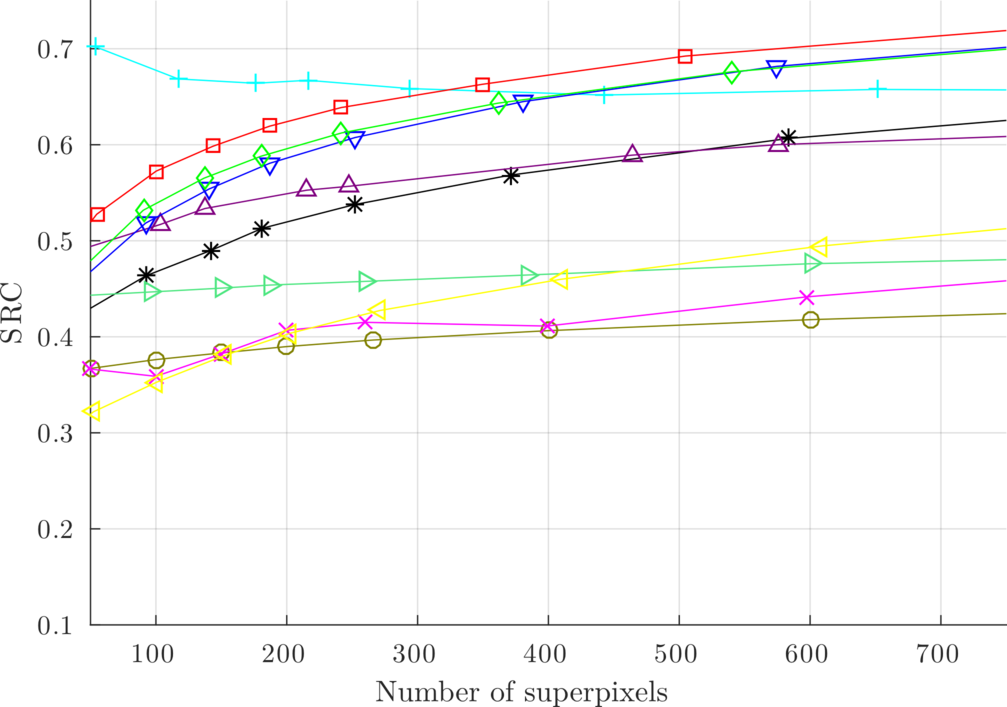}\\
\includegraphics[height=\hhh,width=\www]{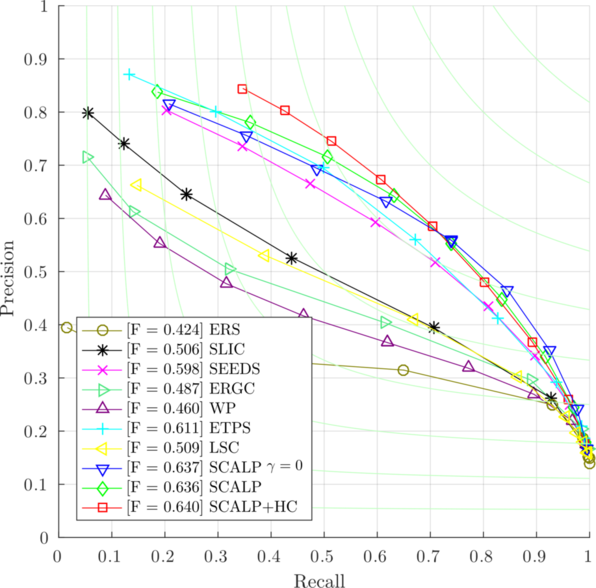}&
\includegraphics[height=\hhh,width=\www]{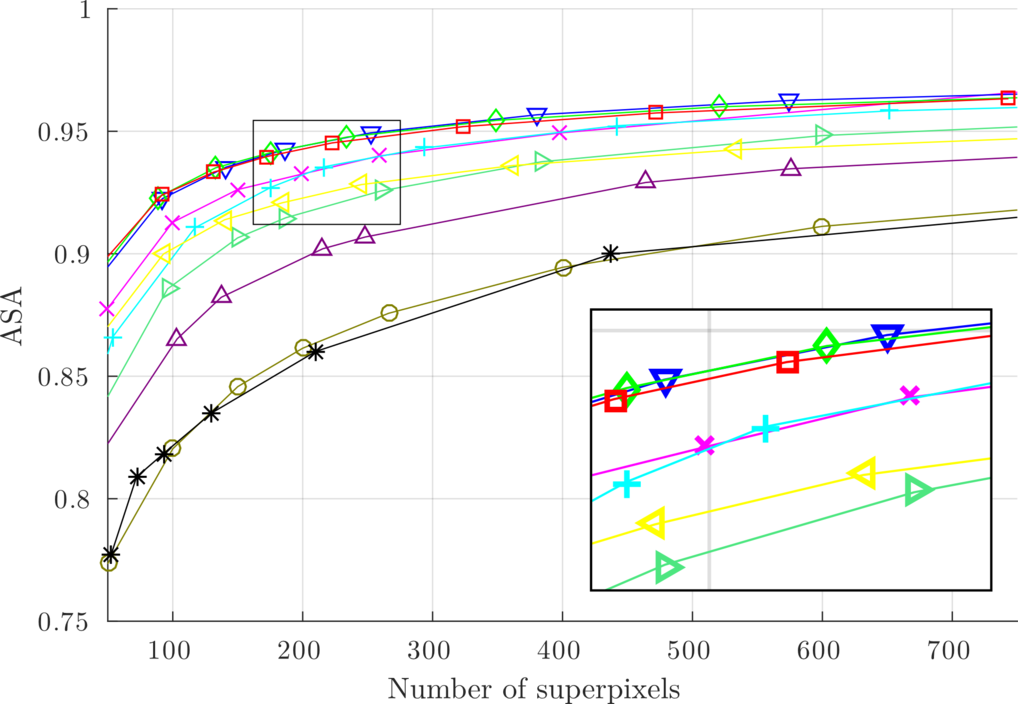}&
\includegraphics[height=\hhh,width=\www]{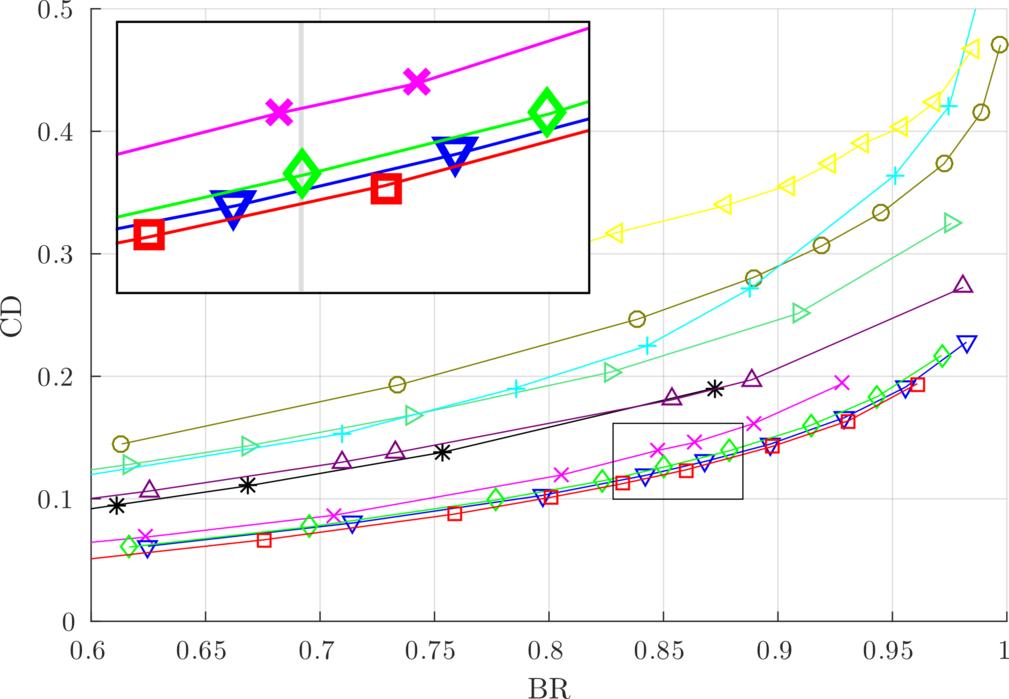}&
\includegraphics[height=\hhh,width=\www]{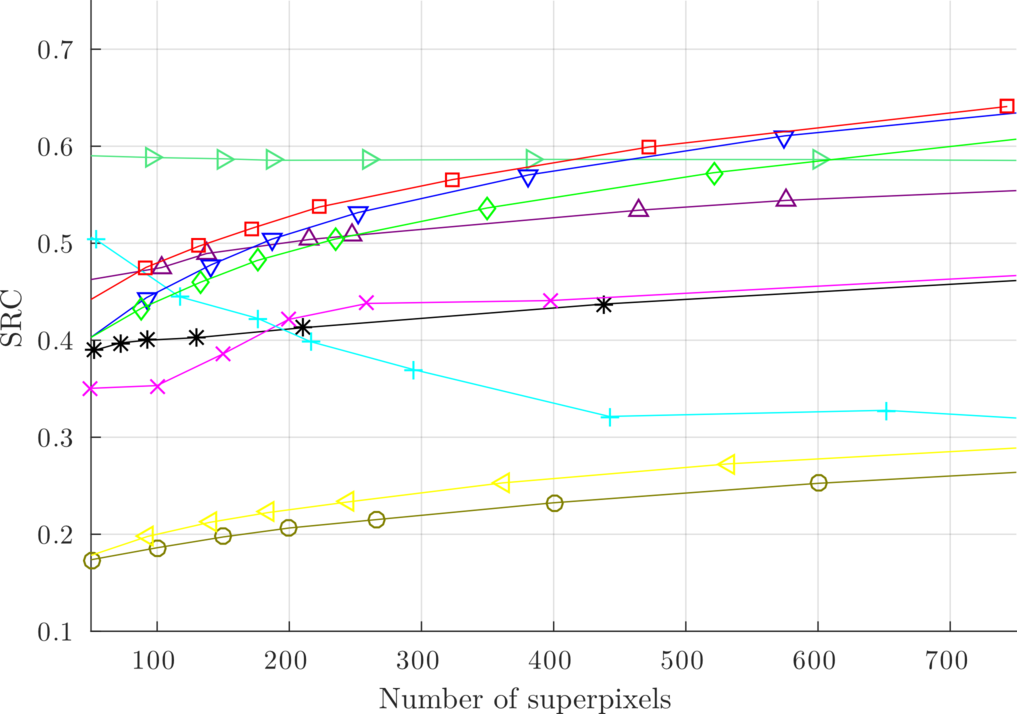}\\
\end{tabular}
} 
\caption{Comparison between the proposed SCALP framework and the state-of-the-art methods 
on contour detection (PR) and superpixel metrics (ASA, CD over BR and SRC) on the BSD test set. 
SCALP outperforms the other methods on both initial images (top) and noisy images (bottom).
Moreover the results obtained with SCALP without using 
a contour prior (SCALP $\gamma=0$), still
outperform the most accurate compared methods \citep{chen2017,liu2011}.
}  
\label{fig:soa_metrics}
\end{figure*}

\subsection{{\label{sec:soa}}Comparison with State-of-the-Art Methods}

We compare the proposed SCALP approach to the following state-of-the-art methods:
ERS \citep{liu2011},
SLIC \citep{achanta2012},
SEEDS \citep{vandenbergh2012},
ERGC \citep{buyssens2014},
Waterpixels (WP) \citep{machairas2015},
ETPS \citep{yao2015} and
LSC \citep{chen2017}.
Reported results are computed 
with codes provided by the authors, in their default settings.

In Figure \ref{fig:soa_metrics},
we provide PR curves with the maximum F-measure,  
and report the standard 
ASA \eqref{asa}, CD \eqref{cd} over BR \eqref{br} and SRC \eqref{circu} metrics
on both initial (top) and noisy (bottom) images.
SCALP outperforms the compared methods on the respect of image objects and contour detection metrics, 
providing for instance higher F-measure ($\text{F}=0.680$), 
while producing regular superpixels.
The regularity is indeed increased compared to SLIC and LSC,
and is among the highest of state-of-the-art methods.
The ASA evaluates the consistency of a superpixel decomposition with respect to
the image objects, enhancing the largest possible overlap.
Therefore, best ASA results obtained with SCALP indicate that the 
superpixels are better contained in the image objects.
Using the contour prior as a hard constraint (SCALP+HC), our method even reaches higher performances, 
for instance
with $\text{F}=0.709$.
Moreover, SCALP results obtained without using 
a contour prior, \emph{i.e.}, setting $\gamma$ to $0$ in \eqref{newweight}, still
outperform the ones of the most accurate compared methods LSC and ERS.
Finally, we can underline the fact that SCALP results outperform the ones of all the compared state-of-the-art methods
on contour detection and respect of image objects metrics while producing regular superpixels.
{\color{black}
The gain of performances is further assessed by the result of a 
paired Student test on the ASA result sets.
A very low $p$-value  ($<0.002$) is obtained by comparing the result set of SCALP to the one 
of ERS \citep{liu2011}, the best compared method in terms of accuracy,
which demonstrates the significant increase of performances obtained with SCALP.}
Generally, to enforce the regularity may reduce the contour adherence \citep{vandenbergh2012}, 
but SCALP succeeds in providing regular but accurate superpixels.
This regularity property has been proven crucial for object recognition \citep{gould2014},
tracking \citep{reso2013} and segmentation and labeling applications \citep{strassburg2015influence}.
Therefore, the use of SCALP may increase the accuracy of such superpixel-based methods.

\newcommand{\tes}{2.5mm}
\begin{table}[t!]
 \caption{
Comparison to state-of-the-art methods on initial$|$noisy images.
The maximum F-measure \eqref{fmeasure} is computed as described in Section \ref{valid}.
ASA \eqref{asa} and SRC \eqref{circu} results are given for $K=250$ superpixels, and
CD \eqref{cd} results for $\text{BR}=0.8$  \eqref{br}.
Blue (bold) and red (underlined) respectively indicate best and second results
}
\renewcommand{\arraystretch}{1}
\begin{center}
{ \footnotesize
 \begin{tabular}{@{\hspace{1mm}}p{1.6cm}@{\hspace{\tes}}c@{\hspace{\tes}}c@{\hspace{\tes}}c@{\hspace{\tes}}c@{\hspace{1mm}}}
 \cline{1-5}
 { Method}& F \eqref{fmeasure}&{ ASA \eqref{asa}} &{ CD/BR \eqref{cd}} &{ SRC \eqref{circu}} \\
 \hline
 { ERS }		
 &$0.593|0.424$	
&$0.951|0.872$	&$0.099|0.227$	&$0.395|0.213$\\ 
 { SLIC  }	
 &$0.633|0.506$	
 	&$0.944|0.867$	&$0.106|0.156$	&$0.537|0.417$\\ 
 { SEEDS  }
 &$0.577|0.598$	
 	&$0.943|0.939$	&$0.109|0.118$	&$0.414|0.435$\\ 
 { ERGC  }	
 &$0.593|0.487$	
	&$0.948|0.924$	&$0.104|0.192$	& $0.457|\mathbf{\color{blue}0.586}$\\ 
 { WP  }	
 &$0.588|0.460$	
 &$0.932|0.907$	&$0.124|0.162$	&$0.557|0.508$\\ 
 { ETPS  }		
 &$0.631|0.509$	
 	&$0.943|0.939$	&$0.110|0.199$	& $\mathbf{\color{blue}0.663}|0.386$\\ 
 { LSC  }		
 &$0.607|0.611$	
 	&$0.950|0.929$	&$0.115|0.300$	&$0.420|0.234$\\ 
 { SCALP}			
 &${\underline{\color{red}0.680}}|{\underline{\color{red}0.636}}$
 &${\underline{\color{red}0.954}|\mathbf{\color{blue}0.949}}$&
 ${\underline{\color{red}0.084}}|\underline{\color{red}0.107}$&$0.614|0.509$\\ 
 { SCALP+HC}&
 $\mathbf{\color{blue}0.709}|\mathbf{\color{blue}0.640}$&
 ${\mathbf{\color{blue}0.955}|\underline{\color{red}0.947}}$&
 ${\mathbf{\color{blue}0.076}|\mathbf{\color{blue}0.101}}$&$\underline{\color{red}0.641}|\underline{\color{red}0.545}$\\ 
 \hline 
  \end{tabular}
 } 
 \end{center}
 \label{table:std}
 \end{table}

The gain over state-of-the-art methods is largely increased when computing superpixels on noisy images.
Methods such as \citet{buyssens2014,chen2017,liu2011} 
obtain very degraded performances
when applied to slightly noised images,
while \citet{vandenbergh2012} is the only method that is robust to noise
on all evaluated aspects.
The state-of-the-art methods can indeed have
very different behavior when applied to noisy images.
They generally produce very noisy superpixel boundaries (see Figure \ref{fig:reg}).
This aspect is expressed by the lower performances of CD over BR in the bottom part of Figure \ref{fig:soa_metrics}.
The regularity is also degraded for all methods, except \citet{buyssens2014}, 
that tends to generate more regular superpixels, failing at grouping homogeneous pixels.
Finally, on the ASA metric, SCALP provides slightly higher results than SCALP+HC for these images.
The presence of noise may mislead the contour detection that should not be considered as a hard constraint 
to ensure the respect of object segmentation.
These results are summarized in Table \ref{table:std}, where we report the performances of all 
compared methods on both initial and noisy images for $K=200$ superpixels.

{\color{black}
Despite the large number of features used in SCALP, the computational time remains reasonable, 
 \emph{i.e.}, less than $0.5$s on BSD images, on a single CPU,
without any multi-threading architecture, contrary to implementations of methods such as ETPS \citep{yao2015}.
This computational time corresponds to standard ones of superpixel methods, 
and SCALP is even faster than methods such as \citet{levinshtein2009,liu2011}, whose computational time 
can be up to $5$s.

In this work, we focus on the decomposition performances and do not extensively compare the processing times, 
since this measurement is highly dependent on the implementation and optimization, and does not necessarily reflect the computational potential of each method
 \citep{stutz2016}.
Nevertheless, our method is based on the iterative clustering framework \citep{achanta2012}, 
and recent works have demonstrated that such algorithm could be implemented to perform in real-time \citep{ban2016,choi2016subsampling,neubert2014compact}.
Therefore, since SCALP have the same  complexity as SLIC, our method can
 reach such computational time with optimized implementation or multi-threading architectures. }

Finally, Figures \ref{fig:soa_images} and \ref{fig:soa_images_noisy} respectively illustrate the superpixel decomposition 
results obtained with SCALP and the best compared methods on initial and noisy images. 
SCALP provides more regular superpixels while tightly following the image contours.
SCALP+HC enables to more accurately guide the decomposition by constraining
superpixels to previously segmented regions.
While most of the compared methods produce inaccurate and irregular results
with slightly noised images (see Figure \ref{fig:soa_images_noisy}), 
  SCALP is robust to noise and produces regular superpixels
  that adhere well to the image contours.

\newcommand{\ppp}{120.3pt} 
\newcommand{\wwwh}{0.1325\textwidth}  
\begin{figure*}[t!]
\centering
{\footnotesize
\begin{tabular}{@{\hspace{0mm}}c@{\hspace{1mm}}c@{\hspace{1mm}}c@{\hspace{1mm}}c@{\hspace{1mm}}c@{\hspace{1mm}}c@{\hspace{1mm}}c@{\hspace{0mm}}}
ERS  &
SLIC  & 
ERGC &
ETPS  &
LSC &
SCALP &SCALP+HC\\
\includegraphics[width=\wwwh]{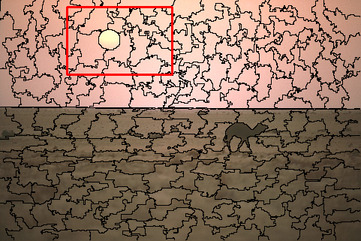}&
\includegraphics[width=\wwwh]{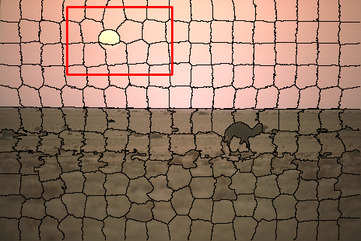}&
\includegraphics[width=\wwwh]{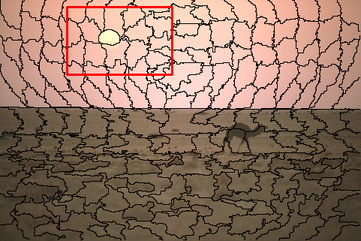}&
\includegraphics[width=\wwwh]{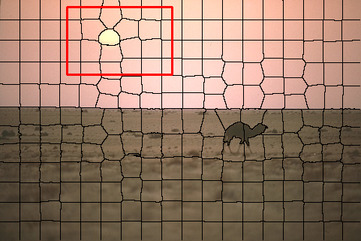}&
\includegraphics[width=\wwwh]{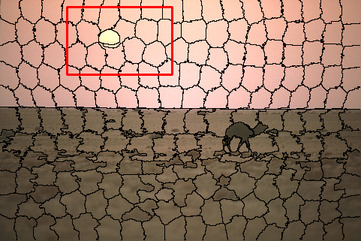}&
\includegraphics[width=\wwwh]{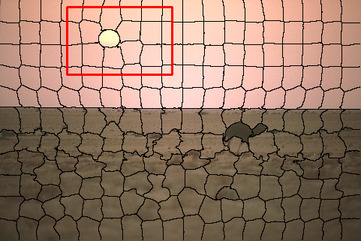}&
\includegraphics[width=\wwwh]{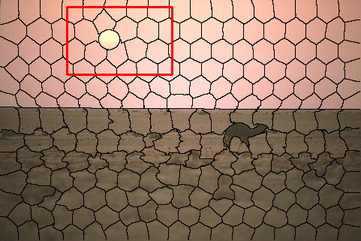}\\ 
\includegraphics[width=\wwwh]{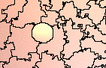}&
\includegraphics[width=\wwwh]{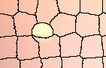}&
\includegraphics[width=\wwwh]{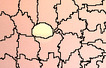}&
\includegraphics[width=\wwwh]{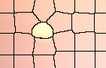}&
\includegraphics[width=\wwwh]{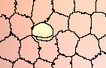}&
\includegraphics[width=\wwwh]{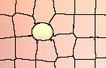}&
\includegraphics[width=\wwwh]{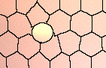}\\
\includegraphics[width=\wwwh]{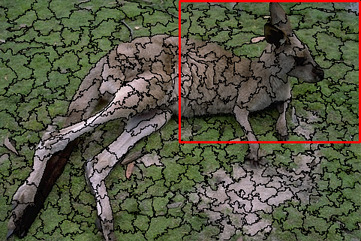}&
\includegraphics[width=\wwwh]{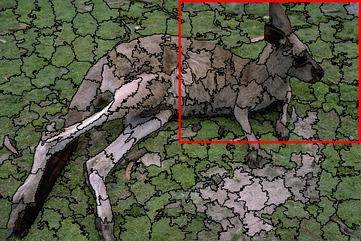}&
\includegraphics[width=\wwwh]{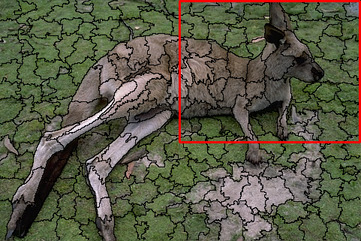}&
\includegraphics[width=\wwwh]{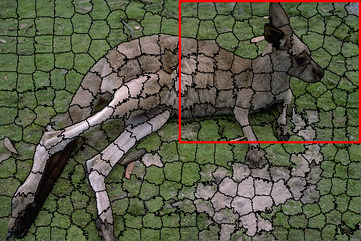}&
\includegraphics[width=\wwwh]{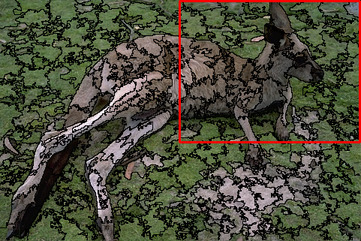}&
\includegraphics[width=\wwwh]{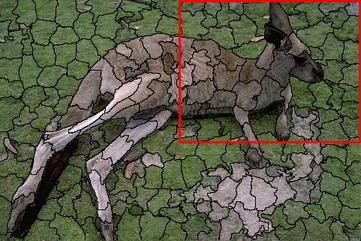}&
\includegraphics[width=\wwwh]{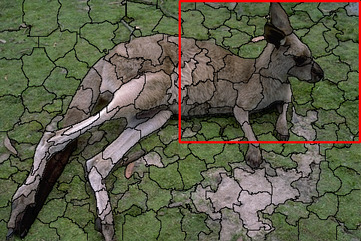}\\ 
\includegraphics[width=\wwwh]{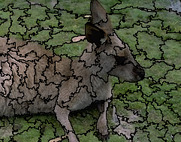}&
\includegraphics[width=\wwwh]{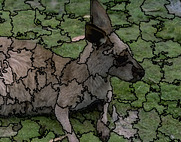}&
\includegraphics[width=\wwwh]{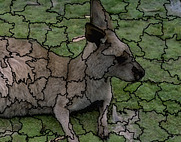}&
\includegraphics[width=\wwwh]{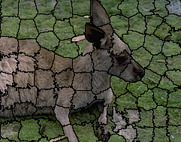}&
\includegraphics[width=\wwwh]{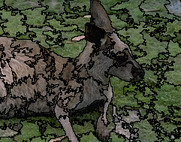}&
\includegraphics[width=\wwwh]{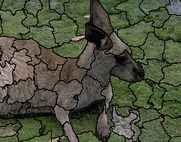}&
\includegraphics[width=\wwwh]{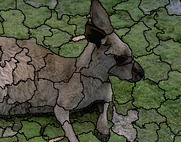}\\
\includegraphics[width=\wwwh]{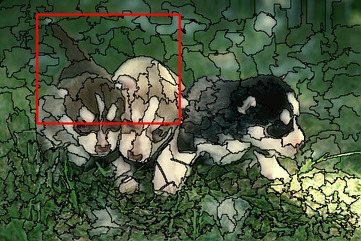}&
\includegraphics[width=\wwwh]{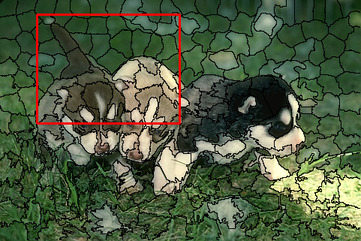}&
\includegraphics[width=\wwwh]{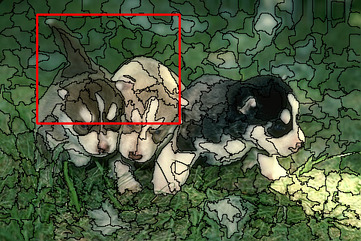}&
\includegraphics[width=\wwwh]{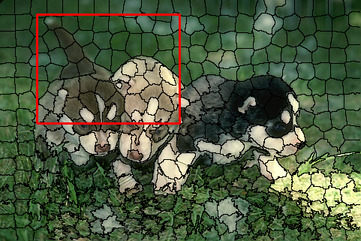}&
\includegraphics[width=\wwwh]{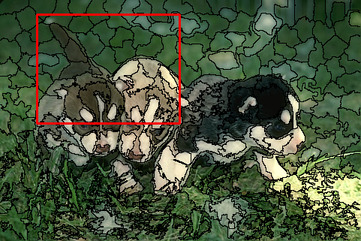}&
\includegraphics[width=\wwwh]{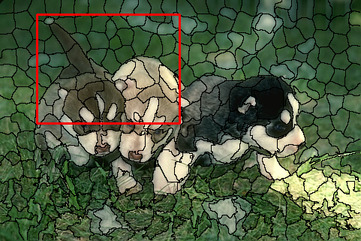}&
\includegraphics[width=\wwwh]{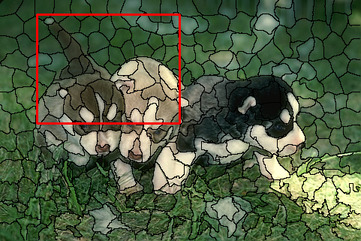}\\
\includegraphics[width=\wwwh]{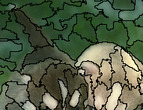}&
\includegraphics[width=\wwwh]{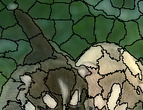}&
\includegraphics[width=\wwwh]{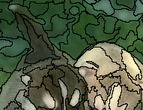}&
\includegraphics[width=\wwwh]{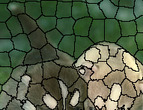}&
\includegraphics[width=\wwwh]{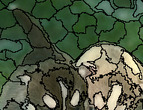}&
\includegraphics[width=\wwwh]{res_112_scalp_zoom.jpg}&
\includegraphics[width=\wwwh]{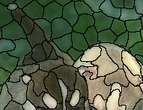}\\
\includegraphics[width=\wwwh]{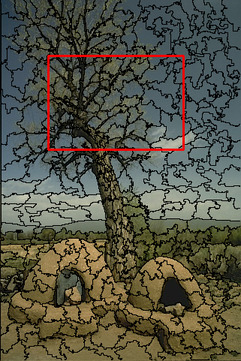}&
\includegraphics[width=\wwwh]{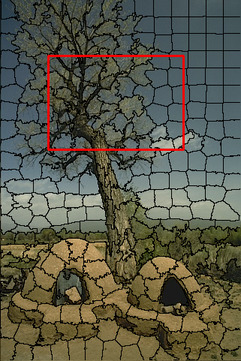}&
\includegraphics[width=\wwwh]{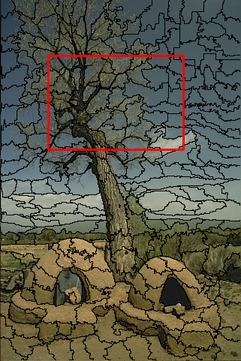}&
\includegraphics[width=\wwwh]{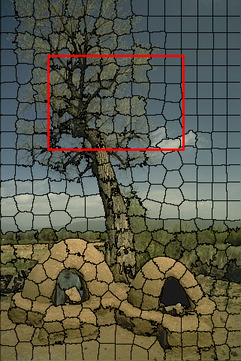}&
\includegraphics[width=\wwwh]{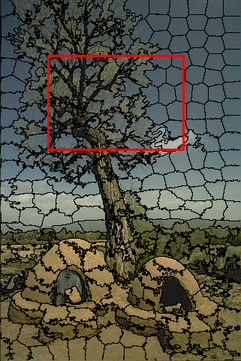}&
\includegraphics[width=\wwwh]{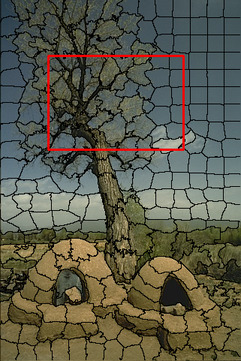}&
\includegraphics[width=\wwwh]{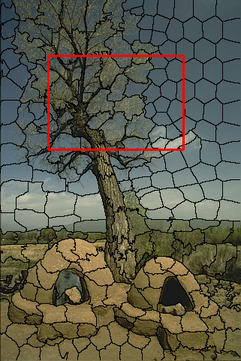}\\ 
\includegraphics[width=\wwwh]{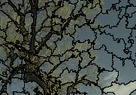}&
\includegraphics[width=\wwwh]{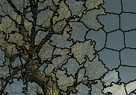}&
\includegraphics[width=\wwwh]{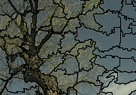}&
\includegraphics[width=\wwwh]{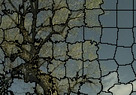}&
\includegraphics[width=\wwwh]{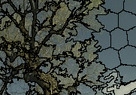}&
\includegraphics[width=\wwwh]{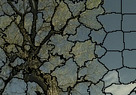}&
\includegraphics[width=\wwwh]{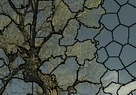}\\
\end{tabular}
} 
\caption{Qualitative comparison of decomposition results between SCALP and state-of-the-art 
superpixel methods on example images of the BSD.
  SCALP provides the most visually satisfying results with superpixels
  that adhere well to the image contours while being equally sized and having compact shapes.
  SCALP+HC enables to further enforce the respect of image contours.
} 
\label{fig:soa_images}
\end{figure*}

\begin{figure*}[t!]
\centering
{\footnotesize
\begin{tabular}{@{\hspace{0mm}}c@{\hspace{1mm}}c@{\hspace{1mm}}c@{\hspace{1mm}}c@{\hspace{1mm}}c@{\hspace{1mm}}c@{\hspace{1mm}}c@{\hspace{0mm}}}
ERS  &
SLIC  & 
ERGC &
ETPS  &
LSC &
SCALP &SCALP+HC\\
\includegraphics[width=\wwwh]{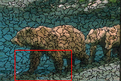}&
\includegraphics[width=\wwwh]{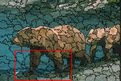}&
\includegraphics[width=\wwwh]{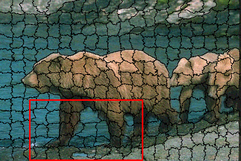}&
\includegraphics[width=\wwwh]{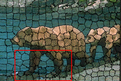}&
\includegraphics[width=\wwwh]{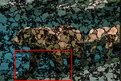}&
\includegraphics[width=\wwwh]{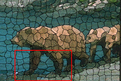}&
\includegraphics[width=\wwwh]{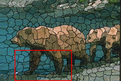}\\ 
\includegraphics[width=\wwwh]{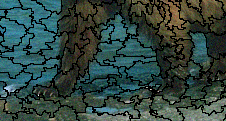}&
\includegraphics[width=\wwwh]{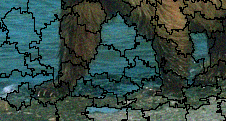}&
\includegraphics[width=\wwwh]{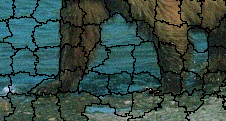}&
\includegraphics[width=\wwwh]{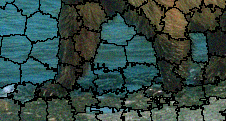}&
\includegraphics[width=\wwwh]{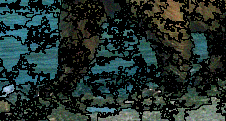}&
\includegraphics[width=\wwwh]{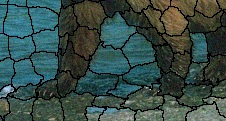}&
\includegraphics[width=\wwwh]{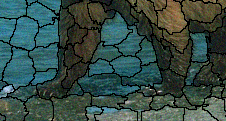}\\
\includegraphics[width=\wwwh]{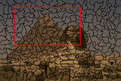}&
\includegraphics[width=\wwwh]{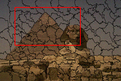}&
\includegraphics[width=\wwwh]{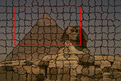}&
\includegraphics[width=\wwwh]{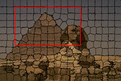}&
\includegraphics[width=\wwwh]{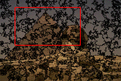}&
\includegraphics[width=\wwwh]{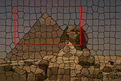}&
\includegraphics[width=\wwwh]{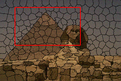}\\ 
\includegraphics[width=\wwwh]{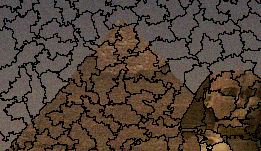}&
\includegraphics[width=\wwwh]{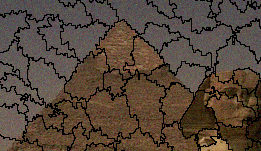}&
\includegraphics[width=\wwwh]{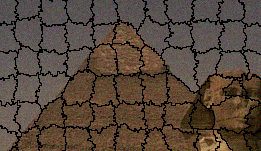}&
\includegraphics[width=\wwwh]{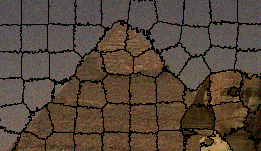}&
\includegraphics[width=\wwwh]{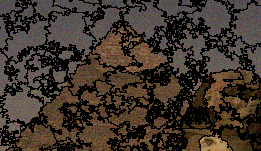}&
\includegraphics[width=\wwwh]{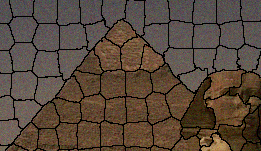}&
\includegraphics[width=\wwwh]{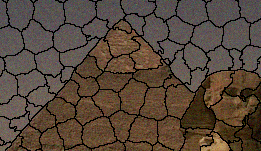}\\
\end{tabular}
} 
\caption{
Qualitative comparison of decomposition results between SCALP and state-of-the-art 
superpixel methods on noisy images of the BSD.
The compared methods generate inaccurate superpixels with noisy borders, while
  SCALP is robust to noise and produces regular superpixels with smooth boundaries
  that adhere well to the image contours. 
  } 
\label{fig:soa_images_noisy}
\end{figure*}

\subsection{Extension to Supervoxels}

Finally, we naturally extend the SCALP method to the computation of supervoxels on 3D volumes, 
for the segmentation of 3D objects or medical images.
Many supervoxel methods are dedicated to video segmentation, see for instance \citet{xu2012evaluation},
and references therein.
These methods segment the volume into temporal superpixel 
tubes and are therefore only adapted to the context of video processing.
Other methods propose to perform superpixel tracking, \emph{e.g.}, \citet{chang2013,reso2013,wang2011},
which can result in similar tubular supervoxel segmentation, 
and may require the computation of optical flow to be efficient \citep{chang2013}.
Contrary to other methods that necessitate substantial adaptations for 3D data, 
we naturally extend SCALP to compute 3D volume decompositions.
We start from a 3D regular grid and perform the decomposition by adding one
dimension to the previous equations presented in Section \ref{sec:scalp}.

To validate our extension to supervoxels, we consider 3D magnetic resonance imaging (MRI) data from 
the Brain Tumor Segmentation (BRATS) dataset \citep{brats2012}.
This dataset contains 80 brain MRI of patients suffering from tumors.
The images are segmented into three labels: background, tumor and edema, surrounding the tumor.
We illustrate examples of SCALP supervoxel segmentation
with the ground truth segmentation  in Figure \ref{fig:brats_ex},
where the tumor and edema are respectively segmented in green and red color.
This dataset is particularly challenging since the resolution of images is very low
and the ground truth segmentation is not necessarily in line with the image gradients.
Finally, note that SCALP obtains an average 3D ASA measure of $0.9848$, and outperforms
state-of-the-art methods with available implementations SLIC \citep{achanta2012} and ERGC \citep{buyssens2014},
that respectively obtain a 3D ASA of  $0.9840$ and $0.9652$.

\newcommand{\wwwhr}{0.14\textwidth}
\newcommand{\wwwhrh}{0.185\textwidth}
\begin{figure*}[t!]
\centering
{\footnotesize
\begin{tabular}{@{\hspace{0mm}}c@{\hspace{1mm}}c@{\hspace{1mm}}c@{\hspace{2.5mm}}c@{\hspace{1mm}}c@{\hspace{1mm}}c@{\hspace{0mm}}}
\includegraphics[width=\wwwhr,height=\wwwhrh]{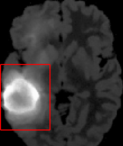}&
\includegraphics[width=\wwwhr,height=\wwwhrh]{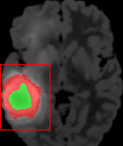}&
\includegraphics[width=\wwwhr,height=\wwwhrh]{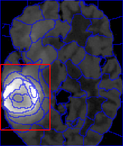}&
\includegraphics[width=\wwwhr,height=\wwwhrh]{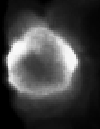}&
\includegraphics[width=\wwwhr,height=\wwwhrh]{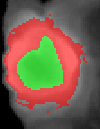}&
\includegraphics[width=\wwwhr,height=\wwwhrh]{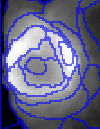}\\ 
\includegraphics[width=\wwwhr,height=\wwwhrh]{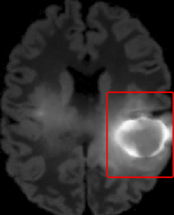}&
\includegraphics[width=\wwwhr,height=\wwwhrh]{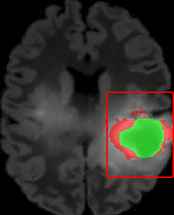}&
\includegraphics[width=\wwwhr,height=\wwwhrh]{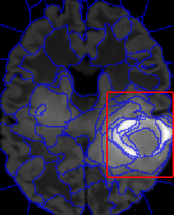}&
\includegraphics[width=\wwwhr,height=\wwwhrh]{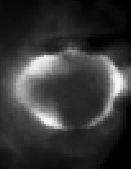}&
\includegraphics[width=\wwwhr,height=\wwwhrh]{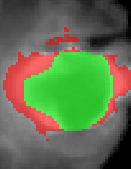}&
\includegraphics[width=\wwwhr,height=\wwwhrh]{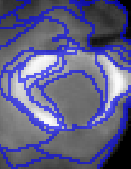}\\ 
Image & Ground truth & Supervoxels&Image & Ground truth & Supervoxels\\
\end{tabular}
} 
\caption{
Examples of supervoxel decomposition with SCALP on medical images \citep{brats2012} (left),
with a zoom (right) on the
tumor and edema, respectively segmented in green and red color.
The structures are detected by SCALP despite the low image resolution.
} 
\label{fig:brats_ex}
\end{figure*}

\section{Conclusion}

 In this work, we generalize the superpixel clustering framework proposed in \citet{achanta2012,giraud2016},
  by considering color features and contour intensity on the linear path from the pixel to the superpixel barycenter.
  Our method is robust to noise and 
  the use of features along such path improves
  the respect of image objects and the shape regularity of the superpixels.
   The considered linear path naturally enforces the superpixel convexity while
   other geodesic distances would provide irregular superpixels.
 Our fast integration of these features within the framework enables 
 to compute the decomposition in a limited computational time.
 SCALP obtains state-of-the-art results, 
  outperforming   the most recent methods of the literature on superpixel and contour detection metrics.
  Image processing and computer vision pipelines would benefit from using 
  such regular, yet accurate decompositions.

\section*{Acknowledgments}
This study has been carried out with financial support from the French 
State, managed by the French National Research Agency (ANR) in the  
frame of the GOTMI project (ANR-16-CE33-0010-01) and the
Investments for the future Program IdEx Bordeaux 
(ANR-10-IDEX-03-02) with the Cluster of excellence CPU.

\bibliographystyle{model2-names}
\bibliography{CVIU}

\end{document}